\documentclass[preprint,12pt]{elsarticle}

\usepackage{latexsym}
\usepackage{graphicx}
\usepackage{multicol,multirow}
\usepackage{amsmath,amssymb,amsfonts}
\usepackage{mathrsfs}
\usepackage{amsthm}
\usepackage{rotating}
\usepackage{appendix}
\usepackage{ifpdf}
\usepackage[T1]{fontenc}
\usepackage{times}
\usepackage{sourcesanspro}
\usepackage{newtxmath}
\usepackage{textcomp}%
\usepackage{xcolor}%
\usepackage{hyperref}

\usepackage{caption}
\usepackage{subcaption}
\usepackage{wrapfig}
\usepackage{makecell}
\usepackage{booktabs}
\newtheorem{prop}{Proposition}

\DeclareUnicodeCharacter{221E}{\ensuremath{\infty}}

\journal{Mechanical Systems and Signal Processing}

\begin{document}

\begin{frontmatter}



\title{Structured Kolmogorov-Arnold Neural ODEs for Interpretable Learning and Symbolic Discovery of Nonlinear Dynamics}


\author[1,3]{Wei Liu}
\author[2,3]{Kiran Bacsa}
\author[1]{Loon Ching Tang}
\author[2,3]{Eleni Chatzi\corref{cor1}}
\cortext[cor1]{Corresponding author at: Department of Civil, Environmental and Geomatic Engineering, ETH Zürich, Zürich, Switzerland}
\ead{chatzi@ibk.baug.ethz.ch}

\affiliation[1]{organization={Department of Industrial Systems Engineering and Management, National University of Singapore},country={Singapore}}
\affiliation[2]{organization={Department of Civil, Environmental and Geomatic Engineering, ETH Zürich},city={Zürich},country={Switzerland}}
\affiliation[3]{organization={Future Resilient Systems, Singapore-ETH Centre},country={Singapore}}

\begin{abstract}
Understanding and modeling nonlinear dynamical systems is a fundamental challenge across science and engineering. Deep learning has shown remarkable potential for capturing complex system behavior, yet achieving models that are both accurate and physically interpretable remains difficult. To address this, we propose \textit{Structured Kolmogorov-Arnold Neural ODEs (SKANODEs)}, a framework that integrates structured state-space modeling with Kolmogorov-Arnold Networks (KANs). Within a Neural ODE architecture, SKANODE employs a fully trainable KAN as a universal function approximator to perform virtual sensing, recovering latent states that correspond to interpretable physical quantities such as displacements and velocities. Leveraging KAN's symbolic regression capability, SKANODE then extracts compact, interpretable expressions for the system's governing dynamics. Experiments on two canonical nonlinear oscillators and a real-world F-16 ground vibration dataset demonstrate that SKANODE reliably recovers physically meaningful latent displacement and velocity trajectories from acceleration measurements, identifies the correct governing nonlinearities—including the cubic stiffness in the Duffing oscillator and the nonlinear damping structure in the Van der Pol oscillator—and reveals hysteretic signatures in the F-16 interface dynamics through structured latent phase portraits and an interpretable symbolic model. Across all three cases, SKANODE provides more accurate and robust predictions than black-box NODE baselines and classical ARX and NARX identification, while producing equation-level descriptions of the learned nonlinear dynamics.
\end{abstract}


\end{frontmatter}



\section{Introduction}
\label{sec1}

The rise of deep learning has significantly advanced the modeling and understanding of complex nonlinear dynamical systems, spanning structural mechanics, fluid dynamics, climate modeling, and biological systems \cite{wang2023scientific,vlachas2022multiscale,hamzi2021learning,li2022deep}. Among the dominant architectures, recurrent neural networks (RNNs) prove notable success in handling sequential data and high-dimensional time series \cite{sherstinsky2020fundamentals,vlachas2024learning}. However, their discrete-time formulation makes them less suitable for capturing the inherently continuous evolution of many real-world systems, especially those governed by differential equations or with irregular temporal sampling \cite{lee2023learning}.

To overcome these limitations, Neural Ordinary Differential Equations (NODEs) have emerged as a principled framework that integrates deep learning with differential equation modeling \cite{chen2018neural}. NODEs offer a flexible, continuous-time latent representation that is more aligned with the physics of real systems. They have been successfully applied across diverse application domains. Yet, despite their expressiveness, NODEs often function as black-box models in that, while they successfully learn efficient latent representations, these do not reveal the underlying (physical) governing mechanisms \cite{norcliffe2020second,fronk2023interpretable,raissi2019physics}. This opacity presents a critical barrier in scientific and engineering disciplines, where trust and understanding of the model are just as important as predictive performance.

To promote interpretability and physical consistency, recent efforts have explored symbolic equation discovery—learning explicit mathematical expressions from data that describe a system’s underlying dynamics \cite{schmidt2009distilling}. A prominent method in this respect is the Sparse Identification of Nonlinear Dynamics (SINDy) framework \cite{brunton2016discovering}, which assumes access to all physical states and their derivatives, and identifies governing equations by performing sparse regression over a predefined function library. SINDy and its extensions \cite{rudy2017data,schaeffer2017learning} have been successfully applied in low-dimensional and noise-free settings. However, they often fail in practical scenarios where system states are only partially observed or measured indirectly and where derivative estimation is highly sensitive to noise.

Other approaches, such as Physics-Informed Neural Networks (PINNs) \cite{raissi2019physics} and PDE-Net \cite{long2019pde}, embed physical knowledge into the training loss to learn governing equations, but typically require known equation structures, spatial derivatives, or collocation points. Genetic programming techniques (e.g., Eureqa \cite{schmidt2009distilling}) and neural-symbolic hybrids (e.g., \cite{chen2022symbolic,sun2023symbolic,kim2020integration}) attempt to evolve symbolic expressions through heuristic search, but often suffer from scalability issues, non-differentiable optimization steps, or dependency on prior knowledge about the functional form.

Most of these existing methods are built on the assumption of direct access to physical coordinates, clean measurements, and a priori specification of basis functions or symbolic templates. As such, their applicability is severely limited in real-world systems where observations are noisy and indirect. This is particularly true in engineering applications, where sensors provide accelerations or other high-order measurements \cite{fritzen2005vibration,hamzi2023learning}. More importantly, most symbolic discovery pipelines are separate from the modeling framework—they operate as post-hoc approximators, rather than as end-to-end trainable components within a system identification architecture.

To overcome these limitations, {\color{black} a novel framework, termed \textit{Structured Kolmogorov-Arnold Neural ODEs (SKANODEs)}, is proposed, which} unifies continuous-time neural modeling, structured physical inductive biases, and \textbf{end-to-end symbolic discovery of governing equations} within a single differentiable pipeline. At the core of SKANODE is the Kolmogorov-Arnold Network (KAN) \cite{liu2024kan}, a recent neural architecture that supports smooth transitions between black-box function approximation and symbolic representation learning. KAN employs adaptive, spline-based basis functions that are differentiable and expressive, while retaining the structure necessary for symbolic interpretability.

In SKANODE, KAN is first used as a \textit{universal function approximator} to learn latent state dynamics that align with physically meaningful quantities such as displacement and velocity. Importantly, these physical states are not assumed to be directly observed. Instead, the model is trained solely on indirect sensor measurements, such as accelerations, by embedding a structured state-space model and observation model that enforce physically consistent relationships between observables and latent coordinates. This enables SKANODE to perform \textit{virtual sensing}, inferring unmeasured physical quantities directly from indirect measurements. Once physically meaningful latent dynamics are revealed, KAN’s \textit{symbolic regression} capability is activated to automatically extract explicit symbolic expressions describing the underlying governing dynamics. The symbolic model is then integrated back into the Neural ODE framework and calibrated to enhance predictive accuracy and improve the precision of the discovered governing equations.

The resulting framework supports both prediction and equation-level interpretation under partial observability. Through extensive evaluations on simulated benchmark systems and a real-world F-16 ground vibration dataset, it is demonstrated that SKANODE achieves strong predictive accuracy while providing transparent, interpretable descriptions of nonlinear dynamics. This capability broadens the practical use of deep learning for scientific modeling, engineering diagnostics, and physics-based decision making.

\section{Background}
\subsection{Physics-Encoded Machine Learning}
Physics-encoded machine learning encompasses a family of approaches that integrate physical principles directly into machine learning architectures, enhancing interpretability, generalization, and predictive performance \cite{faroughi2022physics,haywood2024discussing}. By embedding physical structures within model formulations, these methods provide powerful tools for modeling complex systems and uncovering governing dynamics.

A representative example of this class involves Dynamic Bayesian Networks (DBNs), often instantiated as dynamical variational autoencoders (VAEs) \cite{girin2020dynamical}, where physical knowledge is incorporated into the latent transition dynamics or the encoder-decoder structure. By leveraging such models, prior knowledge about system behavior can be integrated to improve both predictive accuracy and interpretability \cite{liu2022physics,revach2022kalmannet,liu2024neural,takeishi2021physics}. These techniques are closely related to deep state-space models that similarly exploit structured latent representations \cite{rangapuram2018deep,li2019learning}.

Neural Ordinary Differential Equations (NODEs) represent another important class of physics-encoded frameworks. Unlike conventional recurrent neural networks \cite{hochreiter1997long}, which operate in discrete time, NODEs model continuous-time system dynamics by directly learning the governing differential equations. This continuous formulation avoids discretization errors inherent to conventional architectures and offers a natural framework for representing physical processes that evolve continuously in time. Since their introduction, NODEs have been extended to address a variety of structured and domain-specific scenarios \cite{norcliffe2020second,norcliffe2021neural,liu2019neural,xia2021heavy,salvi2022neural}, with applications spanning structural dynamics \cite{najera2023structure}, computational physics \cite{lee2021parameterized}, pharmacology \cite{qian2021integrating}, and chemical engineering \cite{owoyele2022chemnode}. As such an instance, physics-informed NODEs integrate physical knowledge into the model architecture, providing a versatile framework for discrepancy modeling and structural identification of monitored systems \cite{lai2021structural}. As the subsequent sections will detail, our approach is built upon NODEs, leveraging its strengths in capturing complex continuous-time dynamics.

Other related efforts include Hamiltonian-inspired frameworks that enforce energy conservation through tailored network structures \cite{greydanus2019hamiltonian}, often coupled with symplectic integrators to preserve physical invariants \cite{saemundsson2020variational,bacsa2023symplectic}.

\subsection{Neural Ordinary Differential Equations}\label{sec:NODE}
The foundational work by \cite{chen2018neural} introduced Neural Ordinary Differential Equations (NODEs) as a continuous-depth counterpart to discrete neural networks, marking a paradigm shift in how dynamic systems can be approached within the machine learning community. Formally, let $\mathbf{z}(t)\in\mathbb{R}^n$ represent the state of a system at time $t$. The dynamic evolution of the state can be described by the ODE:
\begin{equation}
\frac{d\mathbf{z}(t)}{dt}=f(\mathbf{z}(t),t,\theta),
\end{equation}
where $f:\mathbb{R}^n\times\mathbb{R}\rightarrow\mathbb{R}^n$ is a neural network parameterized by $\theta$, which approximates the derivative of $\mathbf{z}$ with respect to time. Solving this ODE yields the continuous trajectory of $\mathbf{z}(t)$ from an initial state $\mathbf{z}(t_0)$ forward in time.

Training NODEs involves optimizing the parameters $\theta$ to minimize a loss function that measures discrepancies between model predictions and observed data. Unlike conventional neural networks, where backpropagation is applied directly through network layers, NODEs require differentiating through the entire ODE solver, since model predictions are generated by numerically integrating the learned differential equation. This continuous-time backpropagation is efficiently handled using the adjoint sensitivity method \cite{chen2018neural}, which computes gradients with respect to $\theta$ without storing full trajectories, enabling memory-efficient training over long time horizons.

\textbf{Augmented NODEs.}\quad While NODEs are well-suited for modeling continuous dynamics, they are limited in expressivity due to the topological constraints of ODE flows—specifically, their inability to represent functions involving intersecting trajectories or changing data topology. To address this, \cite{dupont2019augmented} introduced Augmented NODEs (ANODEs), which extend the state space by incorporating additional dimensions. By lifting the dynamics into a higher-dimensional space, ANODEs circumvent the representational limitations of standard NODEs. If $\mathbf{q}(t)\in\mathbb{R}^d$ denotes the original system state and $\mathbf{a}(t)\in\mathbb{R}^a$ the augmented components, the combined system evolves as:
\begin{equation}
\frac{d}{dt}\begin{bmatrix}\mathbf{q}(t) \\ \mathbf{a}(t)\end{bmatrix}=f\left(\begin{bmatrix}\mathbf{q}(t) \\ \mathbf{a}(t)\end{bmatrix},t,\theta\right),
\end{equation}
so that the augmented states belong to the higher dimensional space $\mathbb{R}^{d+a}$.

\textbf{Second-order NODEs.}\quad Following a related direction, \cite{norcliffe2020second} proposed Second-Order NODEs (SONODEs), designed to model systems governed by higher-order dynamics, common in many physical and engineering domains. SONODEs are formulated as:
\begin{equation}
\frac{d}{dt}\begin{bmatrix}\mathbf{q}(t) \\ \mathbf{a}(t)\end{bmatrix}=f\left(\begin{bmatrix}\mathbf{q}(t) \\ \mathbf{a}(t)\end{bmatrix},t,\theta\right)=\begin{bmatrix}\mathbf{a}(t) \\ h(\mathbf{q}(t),\mathbf{a}(t),t,\theta)\end{bmatrix}.
\end{equation}
The adjoint state in SONODEs also follows a second-order ODE \cite{norcliffe2020second}, making the adjoint sensitivity method directly applicable. While SONODE introduces an explicit derivative term \( \mathbf{a}(t) \) into the system formulation, the learned states remain abstract latent variables without guaranteed correspondence to physically meaningful quantities such as velocity. In particular, SONODE does not impose structured observation models that reflect the physical measurement processes (e.g., acceleration sensors), which can limit interpretability when applied to real-world physical systems.

\subsection{Symbolic Equation Discovery and Kolmogorov–Arnold Networks}
\label{sec:kan}

Symbolic equation discovery encompasses a family of methodologies aiming to learn explicit, interpretable mathematical expressions governing system behavior, rather than treating the model as a black box. In early approaches, genetic programming-based methods, such as Eureqa \cite{schmidt2009distilling}, searched over expression trees to discover functional relationships \cite{bongard2007automated,tsoulos2006solving}. Although powerful, these methods often suffer from long runtimes, overfitting, and lack of scalability as model complexity grows \cite{YU2025113395}.

More recent techniques such as Sparse Identification of Nonlinear Dynamics (SINDy) \cite{brunton2016discovering} and its extensions \cite{rudy2017data,schaeffer2017learning,pal2025physics} apply sparse regression over pre-defined function libraries to recover concise governing equations. Despite achieving strong performance in controlled scenarios, SINDy-like methods depend heavily on (i) access to full state variables and their derivatives, (ii) carefully selected candidate functions, and (iii) accurate numerical differentiation—assumptions that are often violated in real-world, noisy or partially observed systems.

To improve scalability and expressiveness, differentiable symbolic regression frameworks such as AI Feynman \cite{udrescu2020ai}, PySR \cite{cranmer2023interpretable}, and neural symbolic regression (e.g., \cite{kim2020integration,lample2019deep,kamienny2023deep}) employ neural-guided search or evolutionary optimization to discover analytic expressions. These methods offer greater automation but retain critical limitations: they often function as post-hoc pipelines requiring clean, fully observed data and are rarely integrated directly into architectures like NODEs \cite{YU2025113395}.

Inspired by the quest for scalable, end-to-end symbolic learning, transformer-augmented approaches \cite{biggio2021neural,kamienny2022end,becker2023predicting} and LLM-based frameworks (e.g., LLM-SR \cite{shojaee2025llm}) have recently emerged. These methods leverage pre-trained or generative models to propose symbolic equations, achieving better coverage and robustness. However, they still typically rely on full-state observations and external symbolic modules rather than end-to-end trainable ODE frameworks.

To overcome these limitations, {\color{black} the \textit{Kolmogorov–Arnold Network (KAN)} is adopted}—a neural architecture grounded in the Kolmogorov–Arnold representation theorem and introduced in recent literature \cite{liu2024kan}. KAN replaces traditional MLPs with networks whose activation function nodes are parameterized as learnable spline functions, enabling both:
\begin{itemize}
    \item \textbf{Flexible approximation}: KAN functions as a universal approximator, leveraging spline-based units to learn expressive mappings directly from data;
    \item \textbf{Symbolic extraction}: The learned activation function at each KAN node can be parsed into a compact symbolic form, representing distinct computational components that collectively assemble into the full expression of the governing equation..
\end{itemize}

Recent works \cite{liu2024kan,liu2024kan2} demonstrate KAN’s ability to discover explicit formulas, highlighting its applicability to scientific modeling. Crucially, KANs are fully differentiable and seamlessly integrate into NODE-based architectures. Unlike prior symbolic methods, KAN enables end-to-end symbolic equation discovery from indirect sensor data—such as accelerations—without requiring direct access to physical coordinates or numerical derivatives. This supports the recovery of governing equations under partial observability, making KAN an ideal component for our NODE-based framework.

{\color{black} \textbf{Sparsification.}\quad To support interpretability and generalization, KAN also incorporates a built-in sparsification mechanism. Unlike MLPs, which typically impose sparsity on linear weights, KAN applies L1 regularization directly on the learnable activation functions. Specifically, the L1 norm of an activation function $\phi$ is defined as the average magnitude of its output over a batch of $N_p$ inputs $\{x^{(s)}\}_{s=1}^{N_p}$:
\begin{equation}
|\phi|_1 \equiv \frac{1}{N_p} \sum_{s=1}^{N_p} \left| \phi(x^{(s)}) \right|.
\end{equation}

For a KAN layer $\mathbf{\Phi}$ with $n_{\text{in}}$ inputs and $n_{\text{out}}$ outputs, we denote $\phi_{i,j}$ as the learnable activation function connecting input $i$ to output $j$. The L1 norm of the entire layer is the sum of L1 norms of all activation functions:
\begin{equation}
|\mathbf{\Phi}|_1 \equiv \sum_{i=1}^{n_{\text{in}}} \sum_{j=1}^{n_{\text{out}}} |\phi_{i,j}|_1.
\end{equation}

To further enhance sparsity, KAN introduces an entropy regularization term that encourages selective activation. The entropy of a layer is defined as:
\begin{equation}
S(\mathbf{\Phi}) \equiv - \sum_{i=1}^{n_{\text{in}}} \sum_{j=1}^{n_{\text{out}}} \frac{|\phi_{i,j}|_1}{|\mathbf{\Phi}|_1} \log \left( \frac{|\phi_{i,j}|_1}{|\mathbf{\Phi}|_1} \right),
\end{equation}
which penalizes a uniform distribution over activation magnitudes and encourages a sparse subset of nodes to dominate.

The total training objective combines the prediction loss with both L1 and entropy regularization across all KAN layers indexed by $l = 0, \dots, L-1$:
\begin{equation}
\ell_{\text{total}} = \ell_{\text{pred}} + \lambda \left( \mu_1 \sum_{l=0}^{L-1} |\mathbf{\Phi}^l|_1 + \mu_2 \sum_{l=0}^{L-1} S(\mathbf{\Phi}^l) \right),
\end{equation}
where $\lambda$ controls the overall regularization strength, and $\mu_1, \mu_2$ (typically set to 1) balance the two regularization components.

This sparsification framework allows KAN to learn parsimonious and interpretable structures while preserving representational capacity. In this work, this built-in regularization strategy is adopted as part of the proposed symbolic modeling framework.
}

\subsection{State-Space Representation} \label{sec:ssm}
Transforming higher-order differential equations into a system of first-order equations—known as the state-space representation—has become standard practice across engineering disciplines \cite{kailath1980linear}. This formulation facilitates the analysis and numerical solution of complex dynamical systems by enabling the use of algorithms designed for first-order systems, which are generally more robust and easier to implement \cite{wang2024stablessm}.
Consider, for instance, a second-order nonlinear dynamical system expressed as:
\begin{equation}\label{eq:sode}
\ddot{x}(t)=h(x(t),\dot{x}(t),u(t)).
\end{equation}
Here, $h$ defines the dynamics of single degree-of-freedom (DOF) system, where $x(t)\in\mathbb{R}^1$ represents the displacement and $u(t)\in\mathbb{R}^1$ denotes the input to the system, typically reflecting an external force. In the state-space formulation, the primary focus is on identifying the essential variables that are crucial for completely characterizing the system's state at any particular moment. Typically, these variables include displacements and velocities in second-order physical systems. Thus, the original equation \eqref{eq:sode} can be restructured into two first-order equations, with displacements and velocities serving as the state variables, represented by:
\begin{equation}\label{eq:ssm}
\dot{\mathbf{z}}=
\begin{bmatrix}
\dot{x} \\ \dot{v}
\end{bmatrix}=
\begin{bmatrix}
v \\ h(x,v,u)
\end{bmatrix}
=f(\mathbf{z},u),
\end{equation}
where $\mathbf{z}=[x,v]^T$ encapsulates displacements and velocities. The same principle can be extended to higher-order dynamics by incorporating higher-order derivatives into the state vector. Employing a first-order differential equation simplifies the integration process over time, allowing for use of straightforward methods, such as Euler or Runge-Kutta schemes. This property is further advantageous for applying deep learning methods that are originally invented for first-order equations, such as NODEs, to learn higher-order dynamics.

An interesting observation is that the second first order equation in Eq.\eqref{eq:ssm} precisely describes the generation of acceleration, with the state vector as input to the function $h$. This suggests that when dealing with systems that feature acceleration measurements, the same function $h$ can describe the observation generation process. Consequently, $h$ remains the sole function to be learned, eliminating the need for an additional observation process model to infer the relationship between observation and state vector. Moreover, this also enforces a physically coupled system state and observation, given the derivative relation indicated by Eq.\eqref{eq:ssm}. This implies that the two coordinates and observation adhere to the strict constraint of being interconnected displacement, velocity, and acceleration quantities throughout the training process, enhancing the transparency of the entire framework. Further elaboration on this formulation is provided in the subsequent section, accompanied by a series of theoretical and experimental demonstrations to illustrate its advantages.

It is reminded that we here focus on second-order dynamics because they are prevalent in vibrational analysis, where the system states can be interpreted as displacement and velocity. These dynamics are significant in various engineering and physical systems, such as mechanical oscillators, electrical circuits, and structural vibrations \cite{thompson2002nonlinear}. This approach also serves as a foundation for extending the analysis to more complex, higher-order systems. For completeness, {\color{black} it is noted} that the state-space form naturally applies to higher-order dynamics. Given a general nonlinear $n$-order ODE of the form:
\begin{equation}
\frac{d^nx}{dt^n}=h\left(x,\frac{dx}{dt},\frac{d^2x}{dt^2},...,\frac{d^{n-1}x}{dt^{n-1}},u\right),
\end{equation}
{\color{black} the state variable $\mathbf{z}=[z_1,z_2,z_3,...,z_n]$ is defined as:}
\begin{equation}
z_1=x, z_2=\frac{dx}{dt}, z_3=\frac{d^2x}{dt^2},..., z_n=\frac{d^{n-1}x}{dt^{n-1}}.
\end{equation}
Then the system of state equations can be written as:
\begin{equation}
\dot{\mathbf{z}}=
\begin{bmatrix}
\dot{z}_1 \\ \dot{z}_2 \\ \vdots \\ \dot{z}_{n-1} \\ \dot{z}_n
\end{bmatrix}=
\begin{bmatrix}
z_2 \\ z_3 \\ \vdots \\ z_n \\ h(z_1,z_2,z_3,...,z_n,u)
\end{bmatrix}
=f(\mathbf{z},u).
\end{equation}

\section{Methodology}
\begin{figure}[h]
\centering
\includegraphics[width=1.0\linewidth]{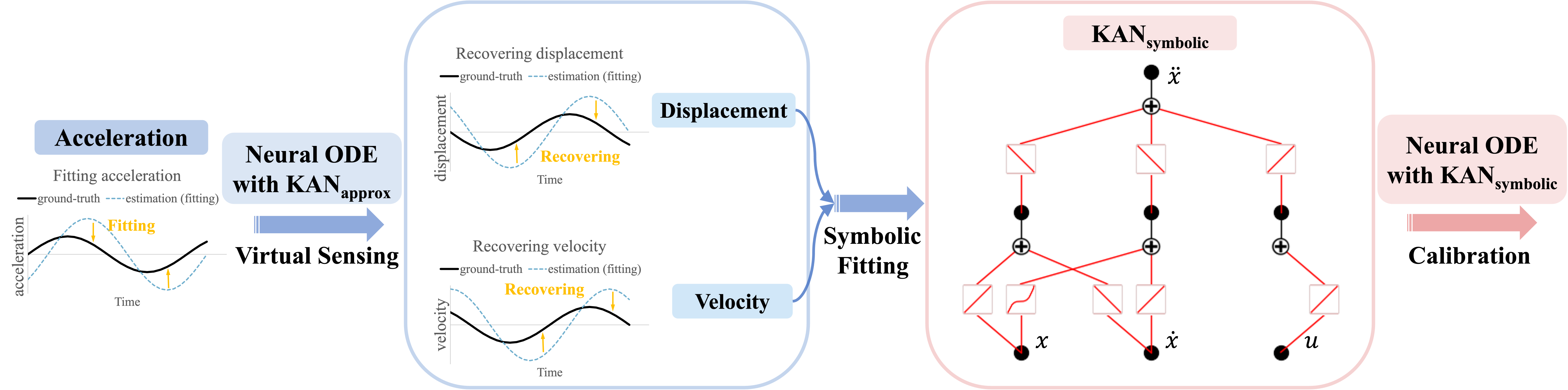}
\caption{A first-stage Kolmogorov--Arnold Network (\( \text{KAN}_{\text{approx}} \)) is employed as a universal function approximator within the proposed structured state-space Neural ODE framework to perform virtual sensing. Throughout training, the latent state variables and reconstructed observations are constrained to evolve as physically interpretable quantities---specifically displacement, velocity, and acceleration---enforced by the inductive biases encoded in the structured representation. Once coherent latent displacement and velocity states are learned, they are passed to a second network (\( \text{KAN}_{\text{symbolic}} \)), which performs symbolic equation discovery and extracts a closed-form expression for the governing dynamics. The resulting symbolic expression learned by (\( \text{KAN}_{\text{symbolic}} \)) is then substituted back into the Neural ODE, and the symbolic model is further trained to calibrate its coefficients, improving both the precision of the discovered governing equation and the predictive accuracy of system responses.}
\label{fig:graph_abs}
\end{figure}

This section presents the proposed \textit{Structured Kolmogorov–Arnold Neural ODE (SKANODE)} framework, which enables \textbf{end-to-end symbolic discovery of governing equations} from partially observed nonlinear dynamical systems. The core objective of SKANODE is to directly learn physically interpretable latent states---such as displacements and velocities---along with explicit symbolic governing equations, using only indirect sensor measurements such as accelerations. To achieve this, SKANODE integrates two key components: (i) a structured state-space formulation that incorporates physics-informed inductive biases by organizing latent states into physically meaningful coordinates, and (ii) a two-stage learning scheme utilizing Kolmogorov–Arnold Networks (KAN) to perform both universal function approximation and symbolic equation extraction. In the first stage, SKANODE employs a fully trainable KAN to approximate the latent acceleration dynamics under the structured Neural ODE formulation. Once physically meaningful latent states are revealed, the dynamics are distilled into compact symbolic forms using the spline-based representation of KAN. The discovered symbolic model is then integrated back into the Neural ODE framework and calibrated to enhance predictive accuracy and improve the precision of the discovered governing equation.

\subsection{Structured State-Space Modeling}

A general deep state-space model for dynamical systems can be formulated as:

\begin{align}
    \dot{\mathbf{z}}(t) &= f_{\theta_t}(\mathbf{z}(t), \mathbf{u}(t)), \label{eq:trans} \\
    \mathbf{y}(t) &= g_{\theta_o}(\mathbf{z}(t), \mathbf{u}(t)), \label{eq:obs}
\end{align}
where \( \mathbf{z}(t) \in \mathbb{R}^n \) denotes the latent state, and \( f_{\theta_t} \), \( g_{\theta_o} \) are learnable functions governing the transition and observation processes, respectively. While expressive, this formulation typically yields black-box models with limited interpretability—particularly for physical systems governed by second-order dynamics.

{\color{black}To address this, a structured state-space formulation is introduced, tailored for modeling such systems using Neural ODEs.} The latent state is defined as:

\begin{equation}
\mathbf{z}(t) =
\begin{bmatrix}
\mathbf{x}(t) \\
\mathbf{v}(t)
\end{bmatrix},
\end{equation}
where \( \mathbf{x}(t) \in \mathbb{R}^d \) represents displacement, and \( \mathbf{v}(t) \in \mathbb{R}^d \) represents velocity for a $d$-dimensional system. The system dynamics are modeled as:

\begin{align}
\frac{d}{dt}
\begin{bmatrix}
\mathbf{x}(t) \\
\mathbf{v}(t)
\end{bmatrix}
&=
\begin{bmatrix}
\mathbf{v}(t) \\
h_\theta(\mathbf{x}(t), \mathbf{v}(t), \mathbf{u}(t))
\end{bmatrix}, \label{eq:ssm_dynamics}
\end{align}
and the observation model is defined by:

\begin{align}
\mathbf{y}(t) = h_\theta(\mathbf{x}(t), \mathbf{v}(t), \mathbf{u}(t)), \label{eq:ssm_obs}
\end{align}
where \( \mathbf{u}(t) \in \mathbb{R}^p \) is the external input (e.g., force), and \( \mathbf{y}(t) \in \mathbb{R}^d \) are the measured accelerations. The same function \( h_\theta \) governs both the transition dynamics and the observation process, ensuring consistency between the latent state evolution and observable measurements. {\color{black} Let \( \mathcal{H} \) denote} the function space of candidate functions for the governing dynamics, consisting of all functions \( h : \mathbb{R}^{d} \times \mathbb{R}^{d} \times \mathbb{R}^{p} \rightarrow \mathbb{R}^{d} \) parameterized by the trainable network \( h_\theta \). In the SKANODE framework, \( \mathcal{H} \) is realized either by fully trainable Kolmogorov–Arnold Networks (KAN) during the universal approximation stage, or by their extracted symbolic representations after symbolic equation discovery.

This formulation implicitly encodes the second-order nature of the physical system. The first equation, \( \frac{d\mathbf{x}}{dt} = \mathbf{v} \), links displacement and velocity, while the second equation models the acceleration dynamics through \( \frac{d\mathbf{v}}{dt} = h_\theta(\mathbf{x}, \mathbf{v}, \mathbf{u}) \). By structuring the observation model directly as acceleration, the framework naturally reflects the physical measurement process often encountered in real-world sensor deployments (e.g., accelerometers).

Although the model is trained solely on acceleration measurements, the structured Neural ODE formulation enforces derivative relationships that allow the latent states to evolve consistently as physically interpretable displacement and velocity trajectories. This tight coupling provides a strong inductive bias that enables the model to recover meaningful latent dynamics even in the absence of direct state measurements.

\textbf{Training Objective.}\quad
To estimate the model parameters \( \theta \), {\color{black} the loss function is defined} as the discrepancy between the predicted and measured accelerations:

\begin{equation}
\mathcal{L}(\theta) = \sum_{t=t_0}^{t_N} \left\| \hat{\mathbf{y}}(t) - \mathbf{y}(t) \right\|^2 = \sum_{t=t_0}^{t_N} \left\| h_\theta\left( \hat{\mathbf{x}}(t), \hat{\mathbf{v}}(t), \mathbf{u}(t) \right) - \mathbf{y}(t) \right\|^2,
\end{equation}
where \( (\hat{\mathbf{x}}(t), \hat{\mathbf{v}}(t)) \) are the estimated latent states obtained by solving the Neural ODE forward in time. The summation is performed over all available time instants \( t_0, t_1, \dots, t_N \), which may be non-uniformly spaced in time.

{\color{black} To avoid overfitting, the sparsification techniques detailed in Section \ref{sec:kan} are incorporated into the training procedure. These include L1 regularization and entropy regularization, which control the model complexity and improve generalization. Additionally, in \( \text{KAN}_{\text{symbolic}} \), several linear nodes are predefined for specific terms, such as the linear components of the damping ratio and stiffness. These linear relationships are fixed apriori, with only the affine weights remaining trainable. Nonlinear terms, such as the cubic term in the Duffing oscillator and the quadratic term in the Van der Pol oscillator, are fully discovered by the network, including both the functional form and the affine parameters.}

\subsection{KAN-Based Two-Stage Learning}

{\color{black}To facilitate interpretable equation discovery within the structured state-space Neural ODE framework, the Kolmogorov--Arnold Network (KAN) is incorporated, which is a spline-based architecture capable of acting both as a universal function approximator and a symbolic expression learner.}

The SKANODE training procedure consists of two key stages:

\paragraph{Stage 1: Learning Structured Latent Dynamics with \( \text{KAN}_{\text{approx}} \)}  

In the first stage, {\color{black} the acceleration dynamics are modeled} using \( \text{KAN}_{\text{approx}} \), a fully trainable Kolmogorov–Arnold Network configured as a universal approximator. This network models the acceleration dynamics based on the structured Neural ODE formulation:
\[
\frac{d}{dt}
\begin{bmatrix}
\mathbf{x}(t) \\
\mathbf{v}(t)
\end{bmatrix}
=
\begin{bmatrix}
\mathbf{v}(t) \\
\text{KAN}_{\text{approx}}(\mathbf{x}(t), \mathbf{v}(t), \mathbf{u}(t))
\end{bmatrix}.
\]
The model is trained using acceleration-only measurements, and due to the enforced derivative relationships in the ODE structure, the latent variables \( \mathbf{x}(t) \) and \( \mathbf{v}(t) \) evolve as displacement and velocity trajectories. The purpose of this stage is to establish a coherent and physically meaningful latent space that reflects the true system dynamics, even under partial observability.

\paragraph{Stage 2: Symbolic Equation Discovery and Calibration with \( \text{KAN}_{\text{symbolic}} \)}  

Once \( \text{KAN}_{\text{approx}} \) has been trained and meaningful latent states are recovered, {\color{black} the learned displacement and velocity signals are used} as inputs to a second KAN, denoted \( \text{KAN}_{\text{symbolic}} \). This KAN is utilized in symbolic mode, enabling automatic extraction of analytic expressions for the system’s governing dynamics.

The symbolic structure discovered by \( \text{KAN}_{\text{symbolic}} \) is then integrated back into the structured ODE framework, replacing \( \text{KAN}_{\text{approx}} \). In this phase, the symbolic form is kept fixed and its affine coefficients are calibrated by further training the Neural ODE with the new symbolic \( \text{KAN}_{\text{symbolic}} \). This calibration process ensures that the symbolic expression remains interpretable while achieving a high-fidelity approximation of both the system response and the true dynamics.

This two-stage learning approach—first leveraging \( \text{KAN}_{\text{approx}} \) to recover physically meaningful latent states, and then using \( \text{KAN}_{\text{symbolic}} \) for symbolic equation discovery and calibration—enables SKANODE to offer a unified, interpretable, and end-to-end differentiable framework for modeling nonlinear dynamical systems. Importantly, the entire process requires only indirect measurements, such as accelerations, eliminating the need for full-state observations or numerical differentiation, which are common limitations in existing symbolic regression methods.

\subsection{Properties of SKANODE}\label{sec:prop}

In general deep state-space models, such as those described by Eqs.~\eqref{eq:trans} and \eqref{eq:obs}, the learned latent state representation is inherently non-unique. Specifically, multiple state-space parameterizations related through admissible coordinate transformations can yield identical observable predictions. This well-known non-identifiability property (e.g., \cite{simon2006optimal}) implies that conventional unconstrained state-space models often produce latent representations that are mathematically valid but physically uninterpretable, as the latent states may correspond to arbitrary linear or nonlinear transformations of the true physical coordinates.

The structured state-space formulation adopted in SKANODE directly addresses this ambiguity by embedding physically meaningful inductive biases into both the latent state design and the observation model. By explicitly structuring the latent space as displacement and velocity, and by defining acceleration as both the second-order state derivative and the observed output, SKANODE tightly constrains the learned representation to evolve consistently with physical quantities of interest. This design significantly reduces the solution space, mitigates coordinate indeterminacy, and promotes the recovery of interpretable system dynamics even when trained solely on indirect measurements such as accelerations.

{\color{black}A formal identifiability result is now stated, establishing that under certain conditions SKANODE can recover the true governing dynamics of the system.}

\begin{prop}[Identifiability of SKANODE]\label{prop:identifiability}
Consider a true second-order dynamical system of the form:
\[
\ddot{x}(t) = h^{\ast}(x(t), \dot{x}(t), u(t)),
\]
where \( h^{\ast} \) denotes the true governing function. Assume:

\begin{itemize}
    \item[(i)] The measured observable is acceleration: \( y(t) = \ddot{x}(t) \).
    \item[(ii)] The structured state-space model given by Eq.~\eqref{eq:ssm_dynamics} and observation model Eq.~\eqref{eq:ssm_obs} are adopted.
    \item[(iii)] The function space \( \mathcal{H} \) consists of functions \( h_\theta \) parameterized by the neural network model, and the true function satisfies \( h^{\ast} \in \mathcal{H} \).
    \item[(iv)] The mappings \( t \mapsto y(t) \) and \( t \mapsto \hat{y}(t) \) both belong to a finite-dimensional function class \( \mathcal{F} \) of dimension \( d_{\mathcal{F}} \).
    \item[(v)] The training data contains \( N > d_{\mathcal{F}} \) distinct time samples \( t_0, \dots, t_N \).
\end{itemize}

Then, if minimizing the loss function
\begin{equation}
\mathcal{L}(\theta) = \sum_{i=0}^{N} \left\| h_\theta\left( \hat{x}(t_i), \hat{v}(t_i), u(t_i) \right) - y(t_i) \right\|^2
\end{equation}
yields zero loss, i.e., \( \mathcal{L}(\theta) = 0 \), it follows that
\[
h_\theta(\hat{x}(t), \hat{v}(t), u(t)) \equiv h^{\ast}(x(t), v(t), u(t)),
\]
and the estimated latent states \( \hat{x}(t), \hat{v}(t) \) coincide with the true states \( x(t), v(t) \) up to numerical integration accuracy.
\end{prop}

This result highlights that, under appropriate structural assumptions, sufficient data richness, and model capacity, the SKANODE framework enables the recovery of both physically meaningful latent states and the true governing dynamics directly from indirect measurements. The structured state-space formulation, combined with consistency between state evolution and observation, eliminates the need for full-state measurements or numerical differentiation, distinguishing SKANODE from conventional black-box approaches and classical symbolic regression methods.

The validity of this proposition is further demonstrated empirically through the simulated and real-world experiments presented in Section~\ref{sec:results}.


\section{Results}\label{sec:results}
To demonstrate the enhanced interpretability and precision of the dynamics learned by the proposed method, {\color{black} SKANODE is assessed} through a series of examples involving both simulated and real-world nonlinear dynamical systems. {\color{black} SKANODE is first applied} to two classic benchmark systems broadly used in nonlinear dynamics research: the Duffing oscillator and the Van der Pol oscillator \cite{thompson2002nonlinear}. These synthetic examples illustrate SKANODE’s ability to accurately recover physically meaningful latent states and to extract precise symbolic expressions for the governing equations. {\color{black} The evaluation is then extended} to real-world data by analyzing ground vibration measurements from an F-16 aircraft, aiming to model the nonlinear interactions at the interface between the aircraft's right wing and its payload. The results across these experiments further support the theoretical properties discussed in the previous section, demonstrating the effectiveness of SKANODE in practical applications.

For all comparative experiments, care was taken to ensure a consistent training setup across models. The only architectural difference lies in the use of the KAN in SKANODE, whereas ANODE and SONODE adopt MLPs. {\color{black}
In addition to deep learning baselines, two classical input-output identification models are also included as references: the linear AutoRegressive model with eXogenous inputs (ARX) and its nonlinear extension (NARX). Both methods are fitted directly on the measured acceleration signals, with the same input-output split as SKANODE, and are evaluated on the same test horizon using identical error metrics. For ARX and NARX, only accelerations are predicted directly; the corresponding displacement and velocity trajectories are obtained by numerical time integration of the predicted accelerations using the ground-truth initial conditions.}

{\color{black} The architectures of the KAN modules used in SKANODE are visualized in the corresponding experimental subsections. In our implementation, B-splines are adopted as the spline basis, with 5th-order B-splines specifically employed. The activation functions are modeled via B-splines with a residual connection, which is a typical design feature of KANs. Each KAN is initialized with 6 knots (forming 5 intervals), and the knots are adaptively updated during the training process.} For MLP-based models, a 5-layer network with 32 hidden units per layer is adopted throughout. Identical data preprocessing, batch sizes, learning rates schedules, and total training epochs were applied across all models to ensure a fair comparison. {\color{black} The experiments are conducted on a local workstation with an NVIDIA GeForce GTX 5080 GPU. Each experiment is run for 1,000 epochs.} The data and code used in this paper are publicly available on GitHub at \url{https://github.com/liouvill/SKANODE}.
\subsection{Duffing Oscillator}
A Duffing oscillator is a representative nonlinear dynamical system describing the motion of a damped oscillator with cubic stiffness nonlinearity. The governing equation is given by:
\begin{equation}\label{eq:duffing}
m\ddot{x}(t) + c\dot{x}(t) + kx(t) + k_3x^3(t) = u(t),
\end{equation}
where $m = 1.0\, \mathrm{kg}$, $c = 0.1\, \mathrm{Ns/m}$, $k = 1.5\, \mathrm{N/m}$, and $k_3 = 10\, \mathrm{N/m}^3$, resulting in strongly nonlinear dynamics. The system is excited by a sinusoidal input $u(t) = \sin(\pi t)$. Data are simulated at a sampling frequency of 10 Hz over 60 seconds, with the first 20 seconds used for training and the remaining 40 seconds for testing.

\begin{figure}[h]
\centering
\includegraphics[width=1.0\linewidth]{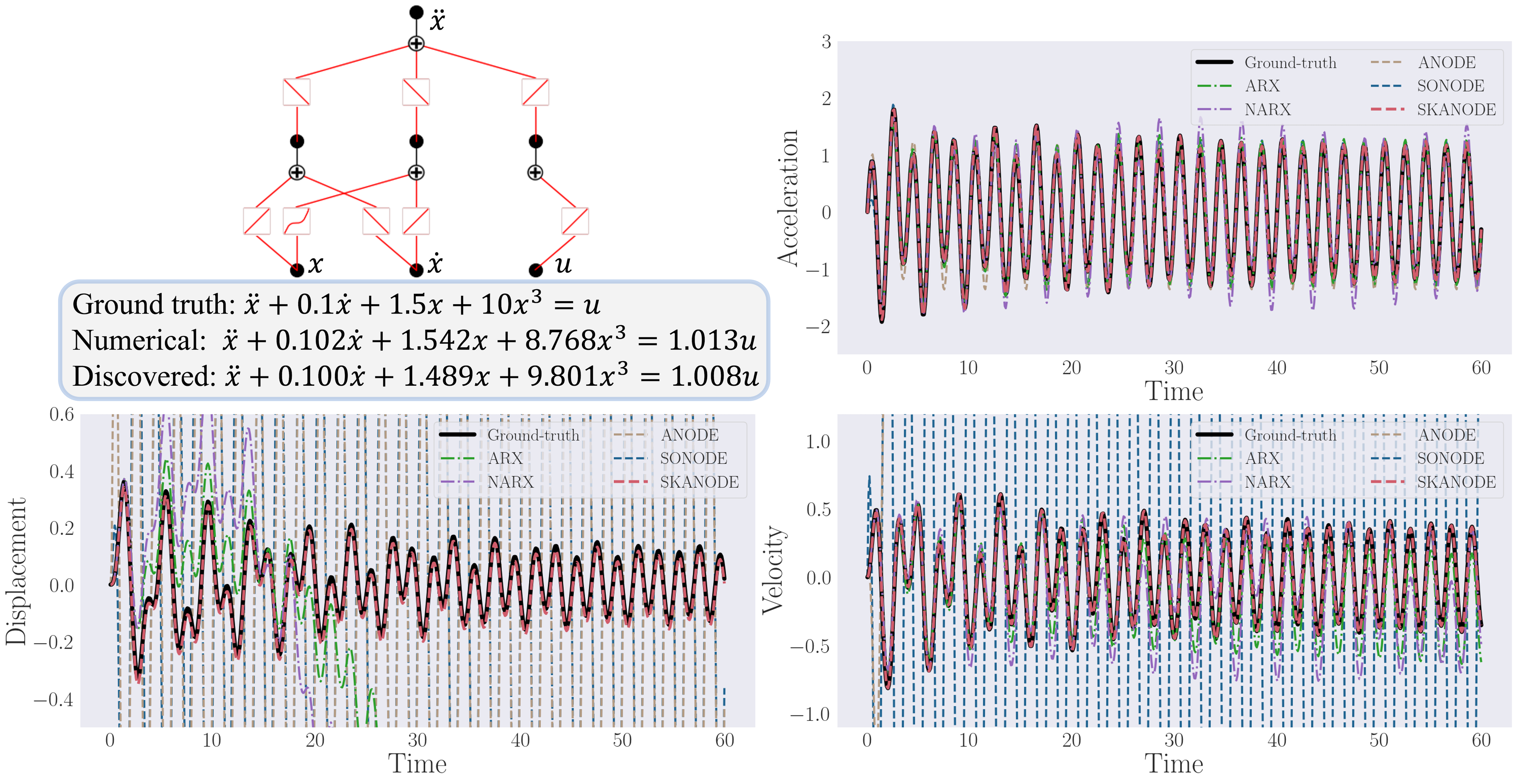}
\caption{Results on the Duffing oscillator. \textbf{Top left:} Symbolic governing equation discovered by $\text{KAN}_{\text{symbolic}}$, compared against the numerical baseline. The identified node receiving displacement input $x$ correctly captures the expected cubic nonlinearity. \textbf{Top right:} Predicted system observables (accelerations) obtained using the SKANODE framework. SKANODE accurately reconstructs accelerations directly from the inferred latent dynamics, without the need for a separate observation model. \textbf{Bottom:} Inferred latent state variables compared to ground truth. The latent states recovered by SKANODE correspond closely to physically meaningful displacement and velocity trajectories, unlike ANODE and SONODE, whose latent states remain abstract representations without clear physical interpretation.}
\label{fig:duffing}
\end{figure}

{\color{black} {\color{black} The performance of SKANODE is evaluated and compared against ANODE and SONODE, as well as classical input-output identification baselines ARX and NARX, on this system, as shown in the top-right subfigure of Figure~\ref{fig:duffing}.} For benchmarking the symbolic learning capability, a numerical baseline is constructed by directly integrating the measured acceleration signals to obtain velocity and displacement, which are then supplied to SINDy to extract a compact symbolic expression. The predicted trajectories from this numerical benchmark are generated by solving the discovered equations and serve to assess the accuracy of the identified dynamics.
It is important to note that this numerical baseline serves solely as a reference for benchmarking the symbolic discovery process; it is not a predictor by itself, as it merely integrates the measured signals without learning the underlying dynamics prior to extracting a symbolic form of the governing equations, and its predictions completely depend on the quality of the extracted equation. In contrast, the proposed SKANODE framework functions both as a dynamics learner—capable of end-to-end prediction of system responses—and as a symbolic extractor that uncovers governing equations from the learned latent quantities.}

The symbolic governing equation extracted by \( \text{KAN}_{\text{symbolic}} \) within SKANODE is shown in Figure~\ref{fig:duffing}, where the identified expression accurately captures the true system dynamics. Notably, the extracted symbolic network exhibits a distinct cubic node connected to the displacement input \( x \), consistent with the characteristic cubic stiffness nonlinearity inherent in the Duffing oscillator. The predicted accelerations and recovered latent state trajectories are also presented. While all three models achieve reasonable accuracy in reconstructing the observable accelerations, only SKANODE successfully recovers latent states that correspond directly to physically meaningful displacement and velocity trajectories. This advantage stems from its structured state-space design and the symbolic learning capability of the Kolmogorov--Arnold Network.

\begin{table}[htbp]
\caption{Model performance measured in MSE for Duffing oscillator experiments subjected to different levels of observation noise. The observational noise is assumed to be Gaussian, randomly sampled from the distribution $\mathcal{N}(0,\sigma^2)$, and the value listed in the noise level column indicates the value of standard deviation $\sigma$. ANODE fails to converge even at the minimal noise level of 0.001.}
\centering
\begin{tabular}{ccccccc}
\hline
& \makecell{Noise\\level} &\vline &Model & \makecell{acceleration\\(measurement)} & displacement & velocity \\
\hline
& &\vline &Numerical 
&3.37$\times10^{-3}$ 
&2.69$\times10^{-3}$ 
&1.06$\times10^{-3}$ \\

& &\vline &ARX 
&1.22$\times10^{-2}$ 
&5.87 
&2.12$\times10^{-2}$ \\

& &\vline &NARX 
&6.82$\times10^{-2}$ 
&6.13 
&4.65$\times10^{-2}$ \\

&0 &\vline &ANODE 
&(3.15~$\pm$~1.42)$\times10^{-2}$ 
&(9.95~$\pm$~0.32)$\times10^{-1}$ 
&6.06~$\pm$~5.17 \\

& &\vline &SONODE 
&(1.52~$\pm$~0.71)$\times10^{-2}$ 
&1.01~$\pm$~0.02 
&9.46~$\pm$~0.14 \\

& &\vline &\textbf{S$^3$NODE} 
&(\textbf{1.11}~$\pm$~\textbf{0.51})$\times\mathbf{10^{-3}}$ 
&(\textbf{1.45}~$\pm$~\textbf{0.76})$\times\mathbf{10^{-3}}$ 
&(\textbf{4.53}~$\pm$~\textbf{2.78})$\times\mathbf{10^{-4}}$ \\

& &\vline &\textbf{SKANODE} 
&(\textbf{3.29}~$\pm$~\textbf{0.00})$\times\mathbf{10^{-5}}$ 
&(\textbf{2.81}~$\pm$~\textbf{0.00})$\times\mathbf{10^{-4}}$ 
&(\textbf{1.31}~$\pm$~\textbf{0.00})$\times\mathbf{10^{-5}}$ \\

\hline
& &\vline &Numerical 
&1.45$\times10^{-3}$ 
&1.12$\times10^{-3}$ 
&4.51$\times10^{-4}$ \\

& &\vline &ARX 
&1.41$\times10^{-2}$ 
&9.68 
&2.97$\times10^{-2}$ \\

& &\vline &NARX 
&1.57$\times10^{-1}$ 
&2.94$\times10^{-1}$ 
&2.74$\times10^{-2}$ \\

&0.001 &\vline &ANODE 
&NaN 
&NaN 
&NaN \\

& &\vline &SONODE 
&(5.41~$\pm$~6.13)$\times10^{-2}$ 
&1.01~$\pm$~0.02 
&9.33~$\pm$~0.33 \\

& &\vline &\textbf{S$^3$NODE} 
&(\textbf{6.78}~$\pm$~\textbf{3.43})$\times\mathbf{10^{-4}}$ 
&(\textbf{1.08}~$\pm$~\textbf{0.49})$\times\mathbf{10^{-3}}$ 
&(\textbf{2.85}~$\pm$~\textbf{1.84})$\times\mathbf{10^{-4}}$ \\

& &\vline &\textbf{SKANODE}  
&(\textbf{3.36}~$\pm$~\textbf{0.00})$\times\mathbf{10^{-5}}$ 
&(\textbf{2.80}~$\pm$~\textbf{0.00})$\times\mathbf{10^{-4}}$ 
&(\textbf{1.30}~$\pm$~\textbf{0.00})$\times\mathbf{10^{-5}}$ \\

\hline
& &\vline &Numerical 
&4.93$\times10^{-3}$ 
&4.48$\times10^{-3}$ 
&1.53$\times10^{-3}$ \\

& &\vline &ARX 
&1.74$\times10^{-2}$ 
&12.49 
&3.81$\times10^{-2}$ \\

& &\vline &NARX 
&1.82$\times10^{-2}$ 
&17.62 
&6.38$\times10^{-2}$ \\

&0.005 &\vline &ANODE 
&NaN 
&NaN 
&NaN \\

& &\vline &SONODE 
&(2.19~$\pm$~2.86)$\times10^{-1}$ 
&(7.92~$\pm$~3.07)$\times10^{-1}$ 
&6.89~$\pm$~3.59 \\

& &\vline &\textbf{S$^3$NODE} 
&(\textbf{1.25}~$\pm$~\textbf{0.99})$\times\mathbf{10^{-3}}$ 
&(\textbf{1.46}~$\pm$~\textbf{1.07})$\times\mathbf{10^{-3}}$ 
&(\textbf{5.09}~$\pm$~\textbf{4.41})$\times\mathbf{10^{-4}}$ \\

& &\vline &\textbf{SKANODE}  
&(\textbf{5.90}~$\pm$~\textbf{0.00})$\times\mathbf{10^{-5}}$ 
&(\textbf{2.90}~$\pm$~\textbf{0.00})$\times\mathbf{10^{-4}}$ 
&(\textbf{1.40}~$\pm$~\textbf{0.00})$\times\mathbf{10^{-5}}$ \\

\hline
& &\vline &Numerical 
&3.04$\times10^{-2}$ 
&1.28$\times10^{-2}$ 
&1.53$\times10^{-2}$ \\

& &\vline &ARX 
&2.34$\times10^{-2}$ 
&44.40 
&9.17$\times10^{-2}$ \\

& &\vline &NARX 
&1.91$\times10^{-2}$ 
&1.52$\times10^{3}$ 
&3.36 \\

&0.01 &\vline &ANODE 
&NaN 
&NaN 
&NaN \\

& &\vline &SONODE 
&(1.76~$\pm$~0.57)$\times10^{-2}$ 
&1.01~$\pm$~0.02 
&9.39~$\pm$~0.22 \\

& &\vline &\textbf{S$^3$NODE} 
&(\textbf{1.18}~$\pm$~\textbf{1.00})$\times\mathbf{10^{-3}}$ 
&(\textbf{4.64}~$\pm$~\textbf{1.97})$\times\mathbf{10^{-4}}$ 
&(\textbf{4.07}~$\pm$~\textbf{4.62})$\times\mathbf{10^{-4}}$ \\

& &\vline &\textbf{SKANODE}  
&(\textbf{1.38}~$\pm$~\textbf{0.00})$\times\mathbf{10^{-4}}$ 
&(\textbf{2.93}~$\pm$~\textbf{0.00})$\times\mathbf{10^{-4}}$ 
&(\textbf{1.44}~$\pm$~\textbf{0.00})$\times\mathbf{10^{-5}}$ \\

\hline
\end{tabular}
\label{tab:duffing}
\end{table}

To quantitatively evaluate model performance, {\color{black} mean squared errors (MSE) are computed} for both predicted accelerations and inferred latent states against the ground truth. The mean and standard deviation of MSE values over three independent runs are reported in Table~\ref{tab:duffing}. All models achieve similar performance in predicting accelerations, while SKANODE consistently shows slightly superior accuracy. More pronounced differences appear in the recovery of latent states: SKANODE achieves the lowest MSE and variance, indicating its ability to extract physically interpretable dynamics reliably across repeated trials. In contrast, ANODE displays high variance and poor latent state interpretability, confirming its tendency to learn arbitrary latent representations. SONODE produces more stable but still non-interpretable latent trajectories, as its structure does not explicitly align with the true physical coordinates. {\color{black} ARX and NARX provide additional reference points from classical system identification. While they can sometimes yield reasonable acceleration predictions, they do not impose physically meaningful state structure and therefore do not support reliable displacement and velocity recovery, especially under measurement noise.
Consistent with this limitation, ARX/NARX exhibit substantially larger errors for latent state reconstruction.} Furthermore, {\color{black} a structured state-space Neural ODE using a conventional multilayer perceptron (MLP) in place of KAN is also evaluated}---referred to as S$^3$NODE in Table~\ref{tab:duffing}---implemented as a 5-layer MLP with 32 hidden units per layer. These results indicate that the advantage of SKANODE stems both from its structured inductive biases in the latent and observation models and from the symbolic discovery capability of KAN, which improves interpretability and predictive accuracy by uncovering governing equations and enhancing the estimation of system quantities.  

\begin{table}[htbp]
\caption{Model performance measured in SSIM for Duffing oscillator experiments subjected to different levels of observation noise. Structural Similarity Index (SSIM) is used to quantify the perceptual similarity between predicted and ground-truth signals. The observational noise is assumed to be Gaussian, randomly sampled from the distribution $\mathcal{N}(0,\sigma^2)$, and the value listed in the noise level column indicates the value of standard deviation $\sigma$. ANODE fails to converge even at the minimal noise level of 0.001.}
\centering
\begin{tabular}{ccccccc}
\hline
& \makecell{Noise\\level} &\vline &Model & \makecell{acceleration\\(measurement)} & displacement & velocity \\
\hline
& &\vline &Numerical 
&0.9759 
&0.5822 
&0.9317 \\

& &\vline &ARX 
&0.9294 
&0.0570 
&0.6204 \\

& &\vline &NARX 
&0.7995 
&0.0388 
&0.4557 \\

& &\vline &ANODE 
&0.8612~$\pm$~0.0471 
&0.0335~$\pm$~0.0010 
&0.0064~$\pm$~0.0013 \\

&0 &\vline &SONODE 
&0.9240~$\pm$~0.0253 
&0.0368~$\pm$~0.0013 
&0.0365~$\pm$~0.0016 \\

& &\vline &\textbf{S$^3$NODE} 
&\textbf{0.9850}~$\pm$~\textbf{0.0083} 
&\textbf{0.6990}~$\pm$~\textbf{0.0847} 
&\textbf{0.9623}~$\pm$~\textbf{0.0209} \\

& &\vline &\textbf{SKANODE}  
&\textbf{0.9997}~$\pm$~\textbf{0.0000} 
&\textbf{0.8689}~$\pm$~\textbf{0.0000} 
&\textbf{0.9974}~$\pm$~\textbf{0.0000} \\

\hline
& &\vline &Numerical 
&0.9873 
&0.7285 
&0.9599 \\

& &\vline &ARX 
&0.9241 
&0.0328 
&0.5355 \\

& &\vline &NARX 
&0.5440 
&0.0271 
&0.4681 \\

&0.001 &\vline &ANODE 
&NaN 
&NaN 
&NaN \\

& &\vline &SONODE 
&0.8623~$\pm$~0.1085 
&0.0366~$\pm$~0.0009 
&0.0370~$\pm$~0.0024 \\

& &\vline &\textbf{S$^3$NODE} 
&\textbf{0.9895}~$\pm$~\textbf{0.0056} 
&\textbf{0.7394}~$\pm$~\textbf{0.0681} 
&\textbf{0.9720}~$\pm$~\textbf{0.0149} \\

& &\vline &\textbf{SKANODE}  
&\textbf{0.9997}~$\pm$~\textbf{0.0000} 
&\textbf{0.8691}~$\pm$~\textbf{0.0000} 
&\textbf{0.9974}~$\pm$~\textbf{0.0000} \\

\hline
& &\vline &Numerical 
&0.9710 
&0.4690 
&0.9173 \\

& &\vline &ARX 
&0.9142 
&0.0177 
&0.4702 \\

& &\vline &NARX 
&0.9132 
&0.0224 
&0.3821 \\

&0.005 &\vline &ANODE 
&NaN 
&NaN 
&NaN \\

& &\vline &SONODE 
&0.6693~$\pm$~0.3518 
&0.0309~$\pm$~0.0093 
&0.0406~$\pm$~0.0053 \\

& &\vline &\textbf{S$^3$NODE} 
&\textbf{0.9845}~$\pm$~\textbf{0.0101} 
&\textbf{0.7171}~$\pm$~\textbf{0.0969} 
&\textbf{0.9605}~$\pm$~\textbf{0.0272} \\

& &\vline &\textbf{SKANODE}  
&\textbf{0.9995}~$\pm$~\textbf{0.0000} 
&\textbf{0.8663}~$\pm$~\textbf{0.0000} 
&\textbf{0.9972}~$\pm$~\textbf{0.0000} \\

\hline
& &\vline &Numerical 
&0.8667 
&0.2673 
&0.6864 \\

& &\vline &ARX 
&0.9037 
&0.0308 
&0.2436 \\

& &\vline &NARX 
&0.8938 
&0.0698 
&0.0616 \\

&0.01 &\vline &ANODE 
&NaN 
&NaN 
&NaN \\

& &\vline &SONODE 
&0.9129~$\pm$~0.0218 
&0.0374~$\pm$~0.0015 
&0.0355~$\pm$~0.0009 \\

& &\vline &\textbf{S$^3$NODE} 
&\textbf{0.9865}~$\pm$~\textbf{0.0120} 
&\textbf{0.8452}~$\pm$~\textbf{0.0458} 
&\textbf{0.9692}~$\pm$~\textbf{0.0291} \\

& &\vline &\textbf{SKANODE}  
&\textbf{0.9990}~$\pm$~\textbf{0.0000} 
&\textbf{0.8654}~$\pm$~\textbf{0.0000} 
&\textbf{0.9971}~$\pm$~\textbf{0.0000} \\

\hline
\end{tabular}
\label{tab:ssim_duffing}
\end{table}

Importantly, SKANODE not only improves interpretability but also contributes to better predictive performance, underscoring the value of incorporating physics-encoded structure and symbolic discovery capability. To further assess robustness, Table~\ref{tab:duffing} presents results under varying levels of observation noise. It is noted that ANODE experiences severe convergence issues under noisy conditions, even at the minimal noise level, whereas SKANODE maintains stable performance across all scenarios.

{\color{black}To complement MSE and provide a more diagnostic perspective, the Structural Similarity Index (SSIM) is also reported. SSIM captures the perceptual similarity between predicted and ground-truth trajectories. Results across different levels of observation noise are summarized in Table~\ref{tab:ssim_duffing}. The SSIM analysis reinforces the conclusions drawn from MSE: SKANODE consistently attains highest scores ($>0.99$ across all states and noise levels), indicating that its predictions are not only numerically accurate but also structurally faithful. S$^3$NODE also demonstrates strong performance but remains consistently below SKANODE, while SONODE yields moderate similarity and ANODE fails to converge under noisy conditions.

\begin{figure}[h]
\centering
\includegraphics[width=1.0\linewidth]{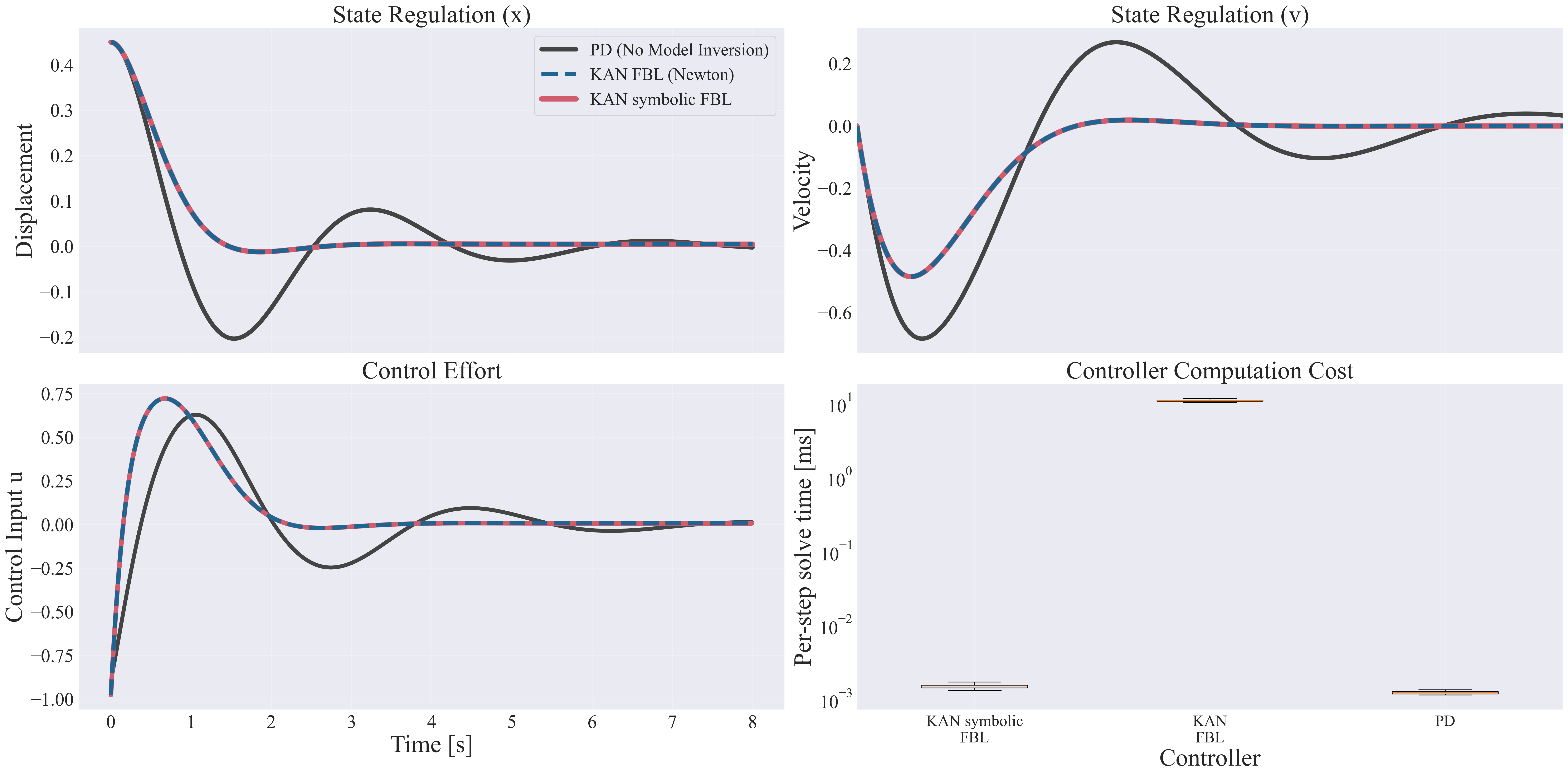}
\caption{Downstream control study on Duffing regulation, comparing PD (no inversion), feedback linearization using the KAN universal-approximator with Newton inversion (KAN FBL), and feedback linearization using the extracted closed-form KAN symbolic model (KAN symbolic FBL).
\textbf{Top}: state regulation trajectories for displacement $x$ (left) and velocity $v$ (right). 
\textbf{Bottom left}: control effort $u(t)$. 
\textbf{Bottom right}: box plot of per-step controller computation time (log scale). The symbolic feedback-linearization controller achieves regulation performance comparable to Newton-based inversion while substantially reducing online computation cost.}
\label{fig:duffing_control}
\end{figure}

{\color{black} To further substantiate that the symbolic representation obtained via the proposed scheme, $\text{KAN}_{\text{symbolic}}$, provides operational value beyond post-hoc interpretability, we introduce an additional downstream control study on the Duffing oscillator. The extracted closed-form governing dynamics are first rewritten in control-affine form and then directly embedded into a classical feedback-linearization (FBL) controller. 

The regulation task is evaluated under observation noise with standard deviation $\sigma = 0.005$, thereby reflecting realistic sensing conditions. Three control strategies are compared: 
(i) feedback linearization using the extracted closed-form symbolic model (KAN symbolic FBL), 
(ii) feedback linearization using the KAN universal approximator, which requires online Newton-based inversion at each time step (KAN FBL), and 
(iii) a baseline proportional–derivative (PD) controller.

This comparison allows us to isolate the practical advantage of the symbolic model: while preserving the nonlinear dynamics learned by the universal approximator, the explicit closed-form expression removes the need for iterative online inversion and enables direct analytical controller synthesis. As shown in Figure \ref{fig:duffing_control}, KAN symbolic FBL and KAN FBL achieve essentially identical regulation performance, attaining RMSE$_x=0.1068$, RMSE$_v=0.1441$, and settling time of $1.57$\,s for both, indicating that the symbolic form precisely preserves the dynamics learned by the universal approximator. Importantly, KAN symbolic FBL is substantially faster at inference time: the mean online solve time is $1.49\,\mu$s compared with $11022.25\,\mu$s for Newton inversion (approximately $7.4\times 10^3$ speedup), with p99 latency $1.75\,\mu$s versus $13511.57\,\mu$s. Compared with PD, KAN symbolic FBL also improves regulation quality (RMSE$_x$: $0.1068$ vs.\ $0.1207$; RMSE$_v$: $0.1441$ vs.\ $0.2285$; settling time: $1.57$\,s vs.\ $5.49$\,s). These results demonstrate that $\text{KAN}_{\text{symbolic}}$ provides a practical downstream benefit: it preserves control performance while dramatically reducing online computational cost, which is critical for real-time model-based control.}

Taken together, these results confirm that the superiority of SKANODE arises from both its structured inductive biases and symbolic discovery capability, which jointly enhance interpretability, predictive accuracy, and robustness compared to conventional neural ODE variants.}

\subsection{Van der Pol Oscillator}\label{sec:vdp}
The Van der Pol oscillator is another canonical nonlinear dynamical system characterized by nonlinear damping, distinct from the Duffing oscillator. Its governing equation is given by:
\begin{equation}\label{eq:vdp}
\ddot{x}(t) + \mu(1 - x^2)\dot{x}(t) + x(t) = u(t),
\end{equation}
where the damping parameter is set to \( \mu = -1 \). The system is excited by a sinusoidal input \( u(t) = \sin(\pi t) \). Data are generated at a sampling rate of 10 Hz over 60 seconds, with the first 20 seconds used for training and the remaining 40 seconds for testing.

\begin{figure}[h]
\centering
\includegraphics[width=1.0\linewidth]{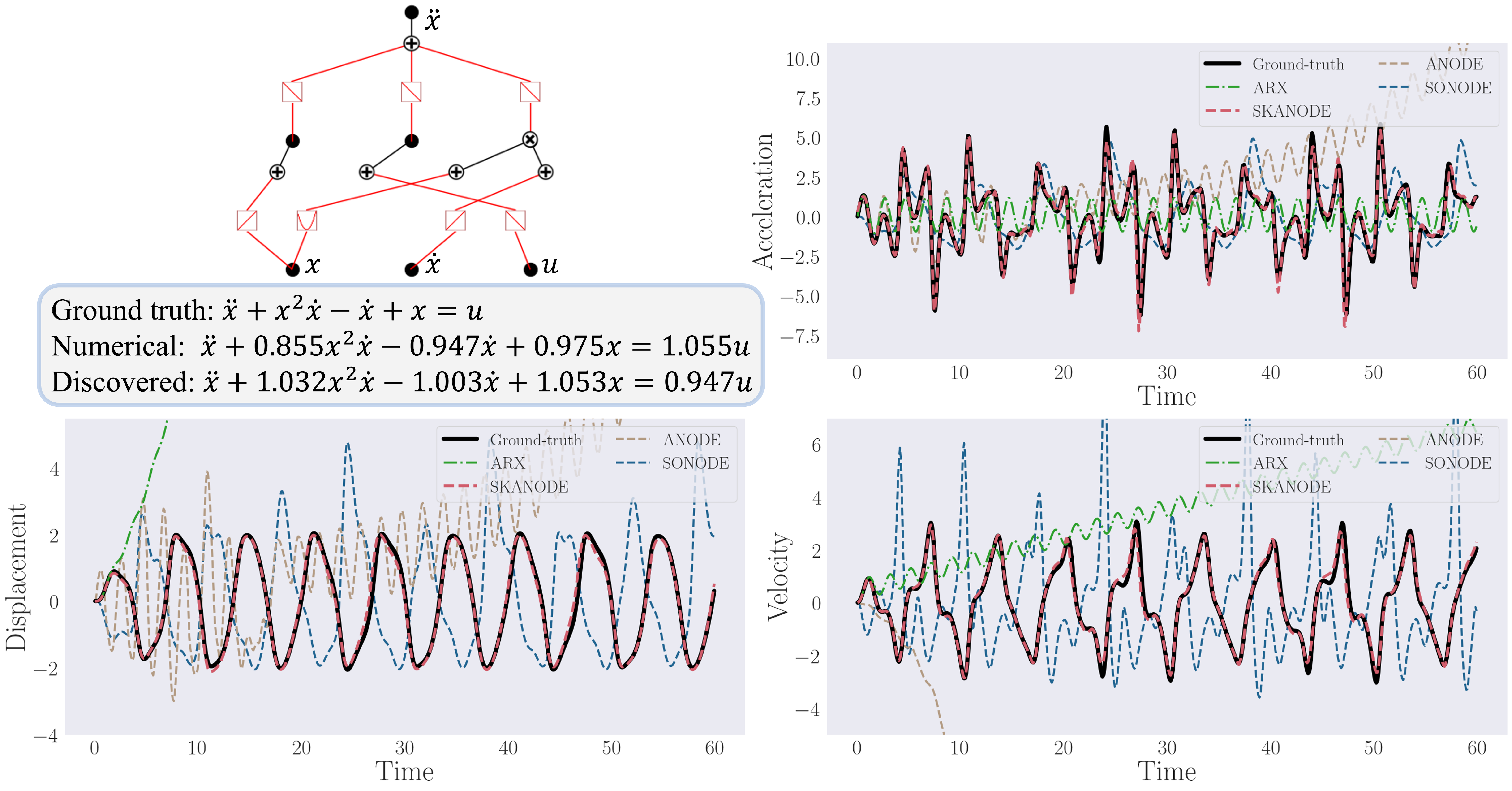}
\caption{Results on the Van der Pol oscillator. NARX trajectories are omitted for visual clarity. \textbf{Top left:} Symbolic governing equation discovered by $\text{KAN}_{\text{symbolic}}$, compared against the numerical baseline. The identified node receiving displacement input $x$ correctly captures the expected quadratic nonlinearity. \textbf{Top right:} Predicted system observables (accelerations) obtained using the SKANODE framework. SKANODE accurately reconstructs accelerations directly from the inferred latent dynamics, without the need for a separate observation model. \textbf{Bottom:} Inferred latent state variables compared to ground truth. The latent states recovered by SKANODE correspond closely to physically meaningful displacement and velocity trajectories, unlike ANODE and SONODE, whose latent states remain abstract representations without clear physical interpretation.}
\label{fig:vdp}
\end{figure}

\begin{table}[h]
\caption{Model performance on Van der Pol (VDP) acceleration prediction measured in MSE and SSIM. SKANODE outperforms both deep learning and numerical baselines.}
\centering
\begin{tabular}{ccc}
\hline
Model & MSE & SSIM \\
\hline
Numerical & 1.0247 & 0.6783 \\
ARX & 3.9180 & 0.0307 \\
NARX & 125.8891 & 0.0045 \\
ANODE & 21.6218 & 0.1612 \\
SONODE & 3.6513 & 0.1502 \\
\textbf{SKANODE} & \textbf{0.2317} & \textbf{0.8764} \\
\hline
\end{tabular}
\label{tab:VDP}
\end{table}

{\color{black} {\color{black}The performance of SKANODE, ANODE, and SONODE, as well as classical identification baselines ARX and NARX, is summarized in Table \ref{tab:VDP}}}. The symbolic governing equation extracted by \( \text{KAN}_{\text{symbolic}} \) within SKANODE is shown in Figure~\ref{fig:vdp}, where the identified expression accurately reflects the true system dynamics, outperforming the numerical baseline. The symbolic network clearly reveals a distinct quadratic node that receives the displacement input \( x \) and multiplies with the velocity input \( \dot{x} \), consistent with the nonlinear damping characteristic of the Van der Pol oscillator. The corresponding predicted accelerations and recovered latent states are also presented. SKANODE accurately reconstructs both the observable accelerations and the latent states, recovering physically meaningful displacement and velocity trajectories. In contrast, ANODE fails to capture the system dynamics entirely, producing large deviations even in reconstructing the measured accelerations. SONODE performs better than ANODE in predicting the observables but still exhibits notable discrepancies and does not recover physically interpretable latent states. Its latent representations remain abstract, as the model lacks explicit physical structure. These results highlight the advantage of SKANODE in simultaneously achieving high predictive accuracy and interpretable latent dynamics through its structured state-space formulation and symbolic discovery capability.

\subsection{F-16 Aircraft}
\begin{figure}[h]
     \centering
     \begin{subfigure}[b]{0.59\textwidth}
         \centering
         \includegraphics[height=0.3\textwidth]{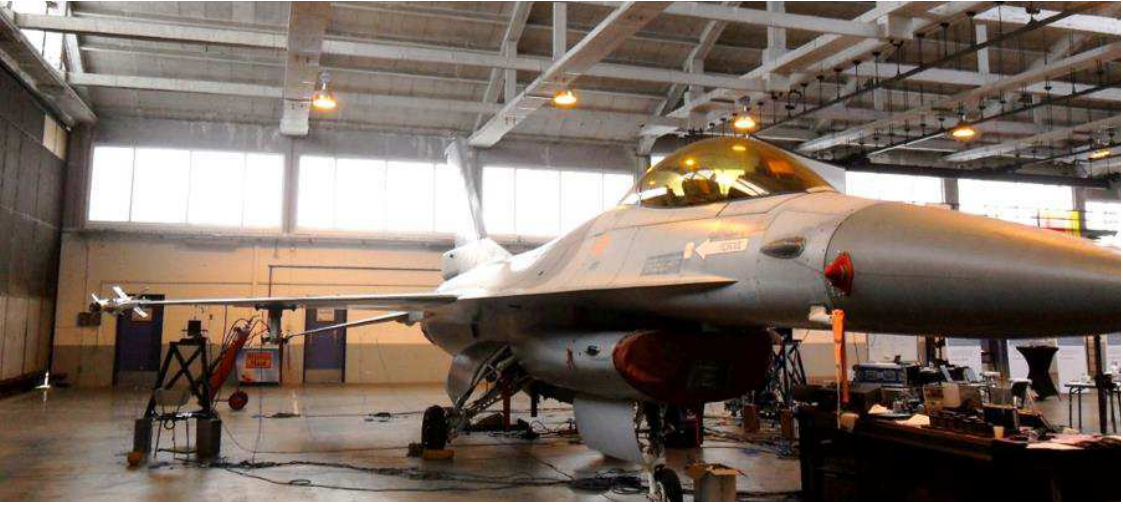}
         \caption{Complete aircraft structure}
         \label{fig:F16_plane}
     \end{subfigure}
     \hfill
     \begin{subfigure}[b]{0.4\textwidth}
         \centering
         \includegraphics[height=0.4\textwidth]{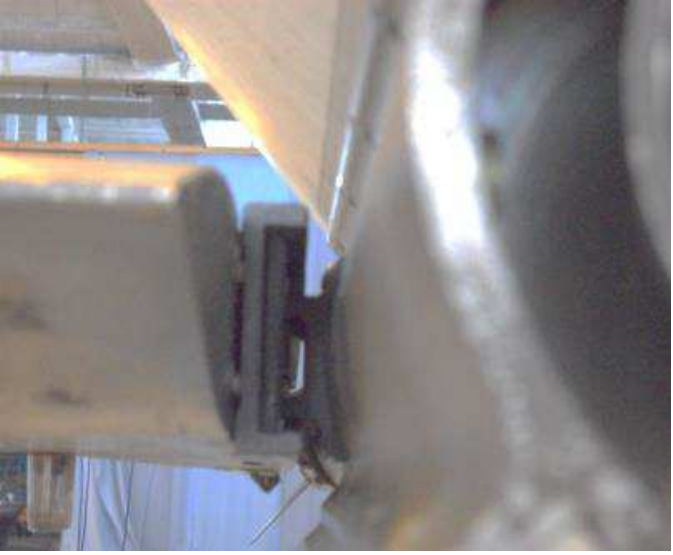}
         \caption{Back connection of the right-wing-to-payload mounting interface}
         \label{fig:F16_interface}
     \end{subfigure}
        \caption{Overview of F16 aircraft and sensor location.}
        \label{fig:F16_photos}
\end{figure}

The last example expands to a real-world complex system. The study presented by \cite{noel2017f} details the experimental data collected from a full-scale F-16 aircraft during a ground vibration test master class. {\color{black} The dataset is publicly available at \url{https://data.4tu.nl/articles/_/12954911}.} To simulate real-world conditions, two dummy payloads were mounted on the wing tips of the aircraft. The setup includes 145 acceleration sensors distributed across the aircraft, under excitation of a shaker installed beneath the right wing to apply sine-sweep excitations over a frequency range of 2 to 15 Hz. Data from three specific sensors were made public: one at the excitation point, one on the right wing near the nonlinear interface of interest, and one on the payload adjacent to the same interface. These measurements were taken at a sampling frequency of 400 Hz. {\color{black}In this work, the highest excitation level in this public dataset (95.6~N input amplitude) is considered, under which the system nonlinearity is most pronounced.} The interfaces, consisting of T-shaped connectors on the payload side, slid through a rail attached to the wing side, were identified by preliminary investigation as primary sources of nonlinearity in the aircraft's structural dynamics, particularly at the right-wing-to-payload interface.

In our setting, the acceleration measured at the excitation point serves as the input reference for the model, while the acceleration measured on the right wing near the nonlinear interface is used as the output signal. Data from the first 2.5 seconds of measurement were used for training, with the subsequent 10 seconds used for testing.

\begin{figure}[h]
\centering
\includegraphics[width=1.0\linewidth]{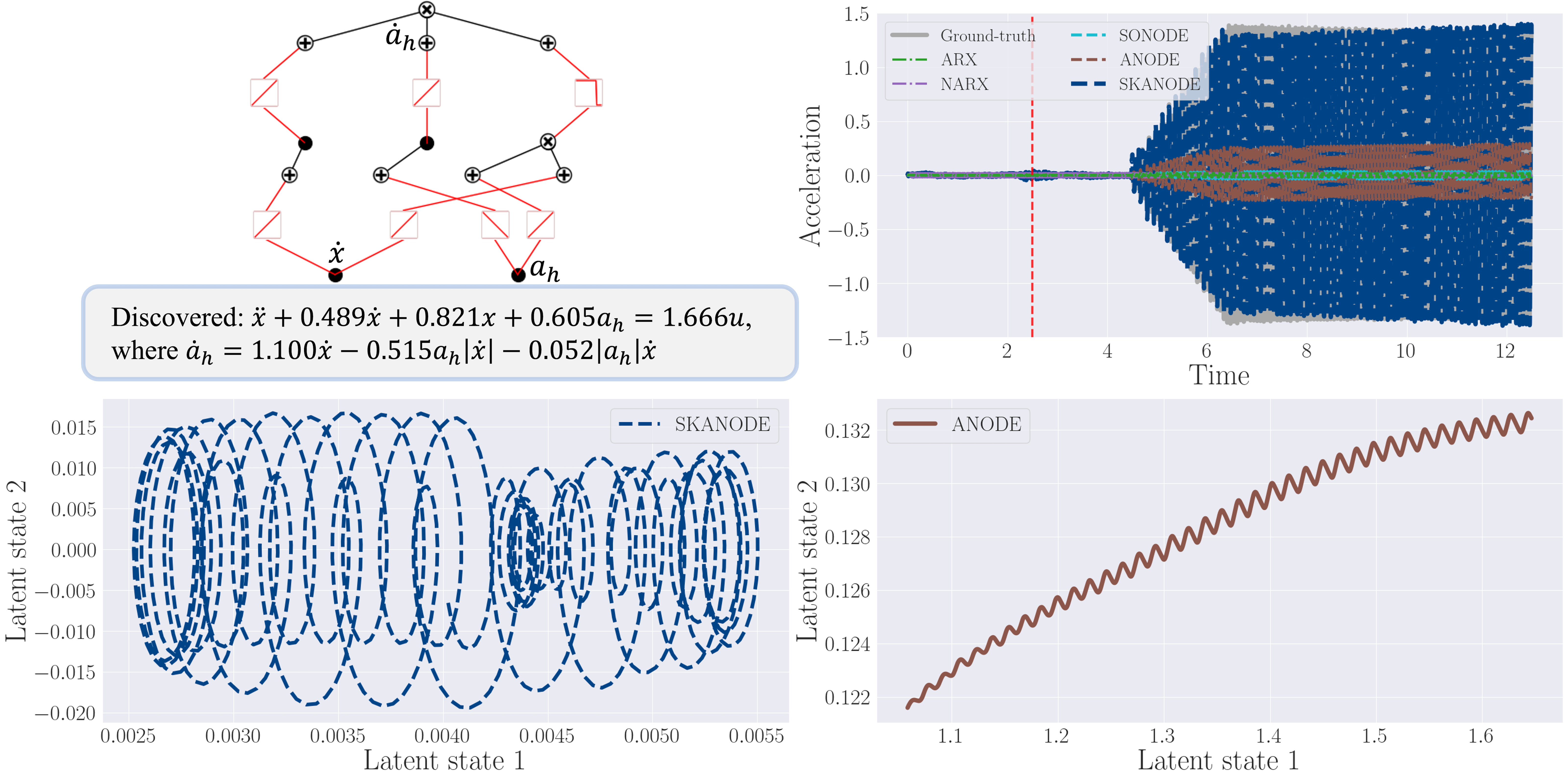}
\caption{Results on the F-16 aircraft. \textbf{Top left:} Symbolic governing equation discovered by \( \text{KAN}_{\text{symbolic}} \). \textbf{Top right:} Predicted system observables (accelerations), with the red vertical line marking the training horizon. \textbf{Bottom:} Inferred latent states visualized in phase portraits. The phase plot obtained by SKANODE exhibits distinct closed-loop patterns characteristic of hysteretic behavior, whereas the phase plot from ANODE yields abstract latent trajectories that lack clear physical interpretation.}
\label{fig:F16}
\end{figure}

\begin{figure}[h]
\centering
\includegraphics[width=1.0\linewidth]{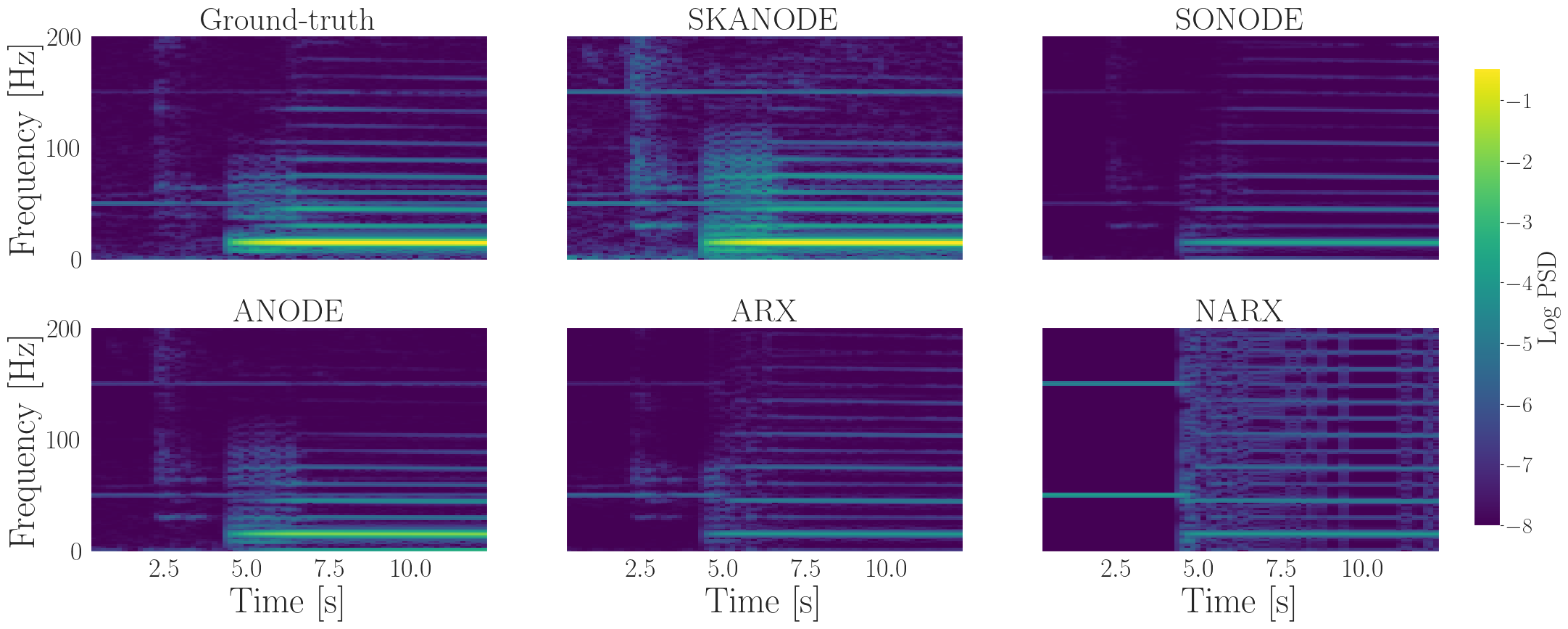}
\caption{Time-frequency comparison of the F-16 acceleration response using spectrograms. SKANODE more faithfully reproduces the dominant frequency bands and their temporal evolution.}
\label{fig:F16_spec}
\end{figure}

The predicted vibration responses are shown in Figure \ref{fig:F16}. {\color{black}Trained only on a short initial segment of the response with substantially smaller amplitude than the subsequent prediction horizon, SKANODE yields a predicted trajectory that matches the measured response significantly better than both the numerical baseline and ANODE.} Furthermore, the phase portraits of the inferred latent states reveal that SKANODE produces a much more structured and physically meaningful phase space, exhibiting characteristic loop patterns indicative of hysteretic behavior. This observation is consistent with the findings of \cite{dossogne2015nonlinear}, confirming the presence of nonlinear hysteresis at the interface. Such interpretability offers important practical value: identifying hysteretic signatures in the dynamics can provide insights into localized energy dissipation mechanisms and evolving structural degradation. This information may support engineers in diagnosing potential fatigue-related issues, monitoring wear at critical joints, and informing predictive maintenance strategies in complex aerospace structures.

{\color{black}In addition, a spectrogram comparison is provided to assess whether the learned models preserve the time-frequency content of the nonlinear response. As shown in Figure~\ref{fig:F16_spec}, SKANODE better preserves the dominant frequency ridges and their time-varying evolution, whereas the baselines exhibit noticeable frequency smearing and mismatched band intensities.}

{\color{black}Building on these observations, SKANODE is extended by introducing an additional latent hysteretic state \( a_h \), leading to the following structured formulation:}
\[
\dot{\mathbf{z}}=
\begin{bmatrix}
\dot{x} \\ \dot{v} \\ \dot{a}_h
\end{bmatrix}=
\begin{bmatrix}
v \\
-kx-\alpha a_h-cv+\frac{1}{m}u \\
h_\theta(v, a_h)
\end{bmatrix},
\]
where \( h_\theta \) captures the hysteretic dynamics and \( \alpha \in [0,1] \) weights the relative contributions of elastic stiffness and hysteretic components.

Through this extended structured state-space design, SKANODE is able to identify a symbolic governing equation, as presented in Figure \ref{fig:F16}. The resulting equation again aligns with the findings of \cite{dossogne2015nonlinear}, confirming that under low-level excitation, the system behavior is dominated by linear stiffness with mild hysteretic effects.

{\color{black} In comparison, the numerical baseline using SINDy fails to identify an equation incorporating the internal hysteretic mechanism. This is because it directly integrates the acceleration signals and seeks an explicit symbolic relation between acceleration, displacement, and velocity, without modeling any latent hysteresis effects. The identified equation from the numerical baseline becomes unstable and diverges over the prediction horizon, as the identified expression does not capture the true underlying dynamics of the system.}

\begin{table}[h]
\caption{Model performance on F-16 acceleration prediction measured in MSE and SSIM. SKANODE outperforms both deep learning and numerical baselines.}
\centering
\begin{tabular}{ccc}
\hline
Model & MSE & SSIM \\
\hline
Numerical & NaN & NaN \\
ARX & 0.5131 & 0.3597 \\
NARX & 0.5085 & 0.3409 \\
ANODE & 0.7139 & 0.3987 \\
SONODE & 0.5324 & 0.3568 \\
\textbf{SKANODE} & \textbf{0.0022} & \textbf{0.9576} \\
\hline
\end{tabular}
\label{tab:F16}
\end{table}

\begin{figure}[h]
\centering
\includegraphics[width=0.55\linewidth]{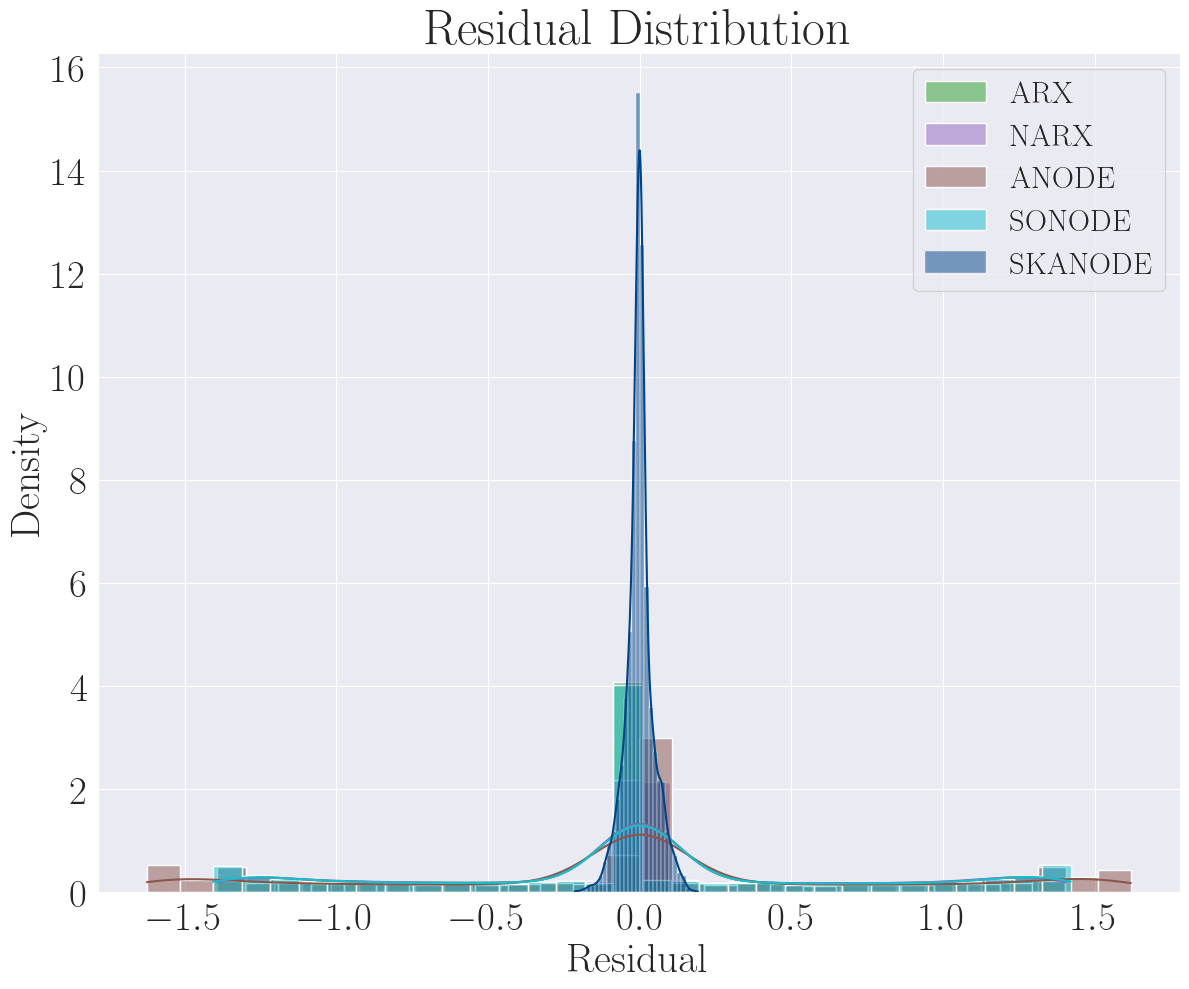}
\caption{Prediction error distribution for the F-16 acceleration response. SKANODE errors are more tightly clustered around zero compared with the deep learning and numerical baselines, indicating higher accuracy and more reliable predictions.}
\label{fig:F16_res}
\end{figure}

{\color{black} We further assess the model performance using both MSE and SSIM metrics. As reported in Table \ref{tab:F16}, the proposed SKANODE consistently outperforms the deep learning as well as the numerical baselines in predicting accelerations. The distribution of prediction errors is revealed in Figure \ref{fig:F16_res}, where SKANODE exhibits errors that are more tightly clustered around zero. This indicates that SKANODE achieves higher accuracy and provides more reliable predictions. Overall, this example underscores the capability of SKANODE to effectively capture complex, real-world nonlinear dynamics through the structured latent state design and interpretable symbolic equation discovery.}

\section{Conclusion}

{\color{black} In this work, a novel framework, termed \textit{Structured Kolmogorov–Arnold Neural ODEs (SKANODE)}, is proposed, which integrates structured state-space modeling with symbolic equation discovery for learning interpretable dynamics from partially observed nonlinear systems.} By embedding a physically meaningful latent state structure and leveraging the Kolmogorov–Arnold Network (KAN) as both a universal function approximator and symbolic learner, SKANODE bridges the gap between the expressive capacity of deep learning and the interpretability required for scientific modeling. This approach enables the direct recovery of physically meaningful latent states, such as displacements and velocities, while simultaneously extracting explicit symbolic governing equations directly from indirect measurements like accelerations.

{\color{black} Comprehensive experiments on multiple nonlinear dynamical systems demonstrate that SKANODE outperforms existing NODE variants as well as classical system identification baselines in both predictive accuracy and interpretability. Specifically, SKANODE recovers equation-level descriptions that match the expected nonlinear mechanisms across all considered systems. For the Duffing oscillator, the discovered symbolic form captures the characteristic cubic stiffness nonlinearity, while the structured latent coordinates evolve as physically meaningful displacement and velocity. For the Van der Pol oscillator, the extracted symbolic structure reflects the nonlinear damping mechanism through the expected state-dependent coupling. For the real-world F-16 ground vibration dataset, the inferred latent phase portraits exhibit closed-loop hysteretic patterns, and the learned symbolic model captures an associated internal hysteretic mechanism via the additional latent hysteretic state. In all three cases, the proposed method yields more accurate and robust predictions than black-box NODE baselines and classical ARX/NARX identification.} These results highlight the importance of combining inductive structural biases with symbolic discovery capabilities to guide the learning process toward solutions that are physically consistent, robust, and interpretable.


The ability to uncover governing equations from indirect observations makes SKANODE particularly well-suited for scientific and engineering domains where full-state measurements are often impractical, and model transparency is critical for trustworthy decision-making. {\color{black} In terms of applicability, the proposed framework is intended for nonlinear dynamical systems whose dominant behavior can be reasonably captured by a smooth ODE that admits a structured state-space representation (e.g., second-order dynamics with displacement and velocity states, or modest extensions with additional latent internal variables). This class includes a wide range of engineering systems in structural dynamics, vibration mitigation, aeroelasticity, and electromechanical devices. Effective use further relies on a measurement model compatible with the imposed structure (e.g., acceleration measurements consistent with second-order dynamics), on sufficient observability of the structured latent states from the available measurements, and on sufficiently informative excitation and, where relevant, known external inputs. When the above conditions are violated---for example in systems dominated by non-smooth switching or impacts, strong sensing distortions such as bias, drift, or colored noise, unmeasured forcing, or weak excitation---the recovered latent states and symbolic expressions may lose physical meaning or become non-unique. In such cases, additional domain-specific constraints, hybrid formulations, or more flexible observation models may be required.

Future work will explore the integration of stronger domain-specific priors and constraints within the symbolic extraction process, as well as extensions that improve robustness to realistic sensing imperfections. One promising extension is to incorporate a partially decoupled observation model trained jointly with the dynamics to account for sensing distortions. Such modeling flexibility would allow the framework to adapt to more challenging real-world sensing environments.}


\section*{Acknowledgement}
The research was conducted at the Future Resilient Systems at the Singapore-ETH Centre, which was established collaboratively between ETH Zurich and the National Research Foundation Singapore. This research is supported by the National Research Foundation Singapore (NRF) under its Campus for Research Excellence and Technological Enterprise (CREATE) programme.

\newpage
\bibliographystyle{elsarticle-num} 
\bibliography{my_references}

\newpage
\appendix
\setcounter{prop}{0}

\section{SKANODE Recovers True System Dynamics}\label{app:prop1}

In this section, {\color{black}it is theoretically demonstrated that, under certain conditions and assuming idealized loss minimization in the theoretical setting}, the proposed Structured Kolmogorov–Arnold Neural ODE (SKANODE) framework is capable of exactly recovering the true governing dynamics.

\begin{prop}[Identifiability of SKANODE]
Consider a true second-order dynamical system of the form:
\[
\ddot{x}(t) = h^{\ast}(x(t), \dot{x}(t), u(t)),
\]
where \( h^{\ast} \) denotes the true governing function. Assume:

\begin{itemize}
    \item[(i)] The measured observable is acceleration: \( y(t) = \ddot{x}(t) \).
    \item[(ii)] The structured state-space model given by Eq.~\eqref{eq:ssm_dynamics} and observation model Eq.~\eqref{eq:ssm_obs} are adopted.
    \item[(iii)] The function space \( \mathcal{H} \) consists of functions \( h_\theta \) parameterized by the neural network model, and the true function satisfies \( h^{\ast} \in \mathcal{H} \).
    \item[(iv)] The mappings \( t \mapsto y(t) \) and \( t \mapsto \hat{y}(t) \) both belong to a finite-dimensional function class \( \mathcal{F} \) of dimension \( d_{\mathcal{F}} \).
    \item[(v)] The training data contains \( N > d_{\mathcal{F}} \) distinct time samples \( t_0, \dots, t_N \).
\end{itemize}

Then, if minimizing the loss function
\begin{equation}
\mathcal{L}(\theta) = \sum_{i=0}^{N} \left\| h_\theta\left( \hat{x}(t_i), \hat{v}(t_i), u(t_i) \right) - y(t_i) \right\|^2
\end{equation}
yields zero loss, i.e., \( \mathcal{L}(\theta) = 0 \), it follows that
\[
h_\theta(\hat{x}(t), \hat{v}(t), u(t)) \equiv h^{\ast}(x(t), v(t), u(t)),
\]
and the estimated latent states \( \hat{x}(t), \hat{v}(t) \) coincide with the true states \( x(t), v(t) \) up to numerical integration accuracy.
\end{prop}

\begin{proof}
The true system evolves according to:
\[
\frac{d}{dt}
\begin{bmatrix} x(t) \\ v(t) \end{bmatrix}
= 
\begin{bmatrix} v(t) \\ h^{\ast}(x(t), v(t), u(t)) \end{bmatrix}, \quad y(t) = h^{\ast}(x(t), v(t), u(t)).
\]

The SKANODE model estimates:
\[
\frac{d}{dt}
\begin{bmatrix} \hat{x}(t) \\ \hat{v}(t) \end{bmatrix}
= 
\begin{bmatrix} \hat{v}(t) \\ h_\theta(\hat{x}(t), \hat{v}(t), u(t)) \end{bmatrix}, \quad \hat{y}(t) = h_\theta(\hat{x}(t), \hat{v}(t), u(t)).
\]

By minimizing the loss to zero, we have \( \hat{y}(t_i) = y(t_i) \) at all training points \( t_i \). Since both \( y(t) \) and \( \hat{y}(t) \) belong to the finite-dimensional function class \( \mathcal{F} \), and we have \( N > d_{\mathcal{F}} \) distinct sampled time points, it follows directly by the definition of \( \mathcal{F} \) that:
\[
\hat{y}(t) \equiv y(t), \quad \forall t.
\]
Therefore, we have
\[
h_\theta(\hat{x}(t), \hat{v}(t), u(t)) = h^{\ast}(x(t), v(t), u(t)) \quad \forall t.
\]

Both systems now evolve under identical ODEs:
\[
\frac{d}{dt} \begin{bmatrix} \hat{x}(t) \\ \hat{v}(t) \end{bmatrix} = \frac{d}{dt} \begin{bmatrix} x(t) \\ v(t) \end{bmatrix},
\]
with identical initial conditions \( \hat{x}(t_0) = x(t_0) \), \( \hat{v}(t_0) = v(t_0) \). By uniqueness of solutions to ODE initial value problems (Picard–Lindelöf theorem), it follows that:
\[
\hat{x}(t) \equiv x(t), \quad \hat{v}(t) \equiv v(t), \quad \forall t.
\]
Thus, both the governing function and the latent states are exactly recovered.
\end{proof}

\end{document}